\documentclass[11pt]{article}
\usepackage{fullpage}

\usepackage{amsmath,amssymb,amsthm,graphicx,bbm,enumerate}
\usepackage{fourier}
\usepackage{verbatim}
\graphicspath{{figures/}}
\def\Reals{\mathbb{R}} 
\def\Naturals{\mathbb{N}} 

\def\sX{{\mathsf X}} 
\def\sU{{\mathsf U}} 

\def\cF{{\mathcal F}} 
\def\cH{{\mathcal H}} 
\def\cK{{\mathcal K}} 
\def\cM{{\mathcal M}} 
\def\cP{{\mathcal P}} 

\def\E{\mathbb{E}} 

\def\cT{{\mathcal T}} 

\def\wh#1{\widehat{#1}} 
\def\bd#1{\boldsymbol{#1}}
\def\ave#1{\langle #1 \rangle} 
\def\Ave#1{\left\langle #1 \right\rangle} 
\def\argmin{\operatornamewithlimits{arg\,min}}

\def\deq{\triangleq}
\def\eps{\varepsilon}

\newtheorem{theorem}{Theorem}
\newtheorem{lemma}{Lemma}
\newtheorem{proposition}{Proposition}
\newtheorem{corollary}{Corollary}
\newtheorem{assumption}{Assumption}
\newtheorem{remark}{Remark}

\usepackage[dvipsnames,usenames]{color}
\usepackage[colorlinks,%
linkcolor=BrickRed,%
filecolor =BrickRed,%
citecolor=RoyalPurple,%
]{hyperref}

\begin{document}

\title{Relax but stay in control: from value to algorithms\\
for online Markov decision processes\thanks{This work was supported by NSF grant CCF-1017564 and by AFOSR grant FA9550-10-1-0390. A preliminary version of this work was presented at the American Control Conference, Portland, OR, June 2014.}}

\author{Peng Guan\thanks{Department of Electrical and Computer Engineering, Duke University, Durham, NC 27708 USA (e-mail: peng.guan@duke.edu).} \and Maxim Raginsky\thanks{Department of Electrical and Computer Engineering and the Coordinated Science Laboratory, University of Illinois at Urbana-Champaign, Urbana, IL 61801 USA (e-mail: maxim@illinois.edu).} \and Rebecca~M.~Willett\thanks{Department of Electrical and Computer Engineering, University of Wisconsin-Madison, Madison, WI 53796, USA; e-mail: willett@discovery.wisc.edu.
}
}

\markboth{}{}

	\maketitle

\begin{abstract}
Online learning algorithms are designed to perform in non-stationary environments, but generally there is no notion of a dynamic {\em state} to model constraints on current and future actions as a function of past actions. State-based models are common in stochastic control settings, but commonly used frameworks such as Markov Decision Processes (MDPs) assume a known stationary environment.
In recent years, there has been a growing interest in combining the above two frameworks and
considering an MDP setting in which the cost function is allowed to change arbitrarily after each
time step. However, most of the work in this area has been algorithmic: given a problem, one
would develop an algorithm almost from scratch. Moreover, the presence of the state and the
assumption of an arbitrarily varying environment complicate both the theoretical analysis and
the development of computationally efficient methods. This paper describes a broad extension of the ideas proposed by Rakhlin et al.\ to give a general framework for deriving algorithms in an MDP setting with arbitrarily changing costs. This framework leads to a unifying view of existing methods and provides a general procedure for constructing new ones. Several new methods are presented, and one of them is shown to have important advantages over a similar method developed from scratch via an online version of approximate dynamic programming.
\end{abstract}

\thispagestyle{empty}

\section{Introduction}

Markov decision processes, or MDPs for short \cite{AC_MDP_survey,Puterman,LermaLasserreMDP}, are a popular framework for sequential decision-making in a dynamic environment. In an MDP, we have states and actions. At each time step of the sequential decision-making process, the agent observes the current state and chooses an action, and the system transitions to the next state according to a fixed and known Markov law. The costs incurred by the agent depend both on his action and on the current state. Traditional theory of MDPs deals with the case when both the transition law and the state-action cost function are known in advance. In this case, there are two ways of designing policies \cite{Bertsekas} -- via dynamic programming (where the construction of an optimal policy revolves around the computation of a relative value function), or via the linear programming (LP) approach \cite{Manne,Borkar}, which reformulates the MDP problem as a ``static'' linear optimization problem over the so-called state-action polytope \cite{Puterman}. However, {\em a priori} known costs are typically unavailable in practical settings. When neither the transition probability nor the cost functions are known in advance, various reinforcement learning (RL) methods, such as the celebrated $Q$-learning algorithm \cite{Watkins_Dayan_QLearning,Tsitsiklis_QLearning} and its variants, can be used to learn an optimal policy in an online regime. However, the key assumptions underlying RL are that the agent is operating in a stochastically stable environment, and that the state-action costs (or at least their expected values with respect to any environmental randomness) do not vary with time. In this paper, instead of considering a fixed or stochastic cost function, we study Markov decision processes where the cost functions are chosen arbitrarily and allowed to change with time. More specifically, we are interested in the {\em online MDP} problem:  just as in the usual online leaning framework \cite{Robbins_compound, Hannan, PLG}, the one-step cost functions form an arbitrarily varying sequence, and the cost function corresponding to each time step is revealed to the agent after an action has been taken. The objective of the agent is to minimize regret relative to the best stationary Markov policy that could have been selected with full knowledge of the cost function sequence over the horizon of interest. The assumption of arbitrary time-varying cost functions makes sense in highly uncertain and complex environments whose temporal evolution may be difficult or costly to model, and it also accounts for collective (and possibly irrational) behavior of any other agents that may be present. The regret minimization viewpoint then ensures that the agent's {\em online} policy is robust against these effects.

Online MDP problems can be viewed as {\em online control problems}. The online aspect is due to the fact that the cost functions are generated by a dynamic environment under no distributional assumptions, and the agent learns the current state-action cost only after committing to an action. The control aspect comes from the fact that the choice of an action at each time step influences future states and costs. Taking into account the effect of past actions on future costs in a dynamic distribution-free setting makes online MDPs hard to solve. To the best of our knowledge, only a few methods have been developed in this area over the past decade \cite{McMahan, EvenDar,Yu, onlineMDP_bandits, AroraTewari,onlineMDP_full,Yadkori,Zimin,DickCsaba}. Most research in this area has been algorithmic: given a problem, one would present a method and prove a guarantee (i.e., a regret bound) on its performance. There are two distinct lines of methods: the algorithms presented by \cite{EvenDar,Yu, onlineMDP_bandits} require the computation of relative value functions at each time step, while the algorithms in \cite{Zimin,DickCsaba} reduce the online MDP problem to an online linear optimization problem and solve it by online learning methods. These two lines of methods correspond to the  two above-mentioned different ways of designing polices for MDPs. From a theoretical and conceptual standpoint, it is desirable to provide a unifying view of existing methods and a general procedure for constructing new ones. In this paper, we present such a general framework for online MDP problems that subsumes the above two approaches. This general framework not only enables us to recover known algorithms, but it also gives us a generic toolbox for deriving new algorithms from a more principled perspective rather than from scratch. 

The online MDP setting we are considering was first defined and studied in the work of \cite{EvenDar} and \cite{Yu}, which deals with MDPs with arbitrarily varying rewards. Like these authors, we assume a full information feedback model and known stochastic state transition dynamics. (However, it should be pointed out that these assumptions have been relaxed in some recent works --- for example, \cite{onlineMDP_bandits} and \cite{AroraTewari} assume only bandit-type feedback, while \cite{Yadkori} prove regret bounds for MDPs with arbitrarily varying transition models and cost functions. An extension of our framework to these settings is an interesting avenue for future research.)

Our general approach is motivated by recent work of Rakhlin et al. \cite{RakhlinRL}, which gives a principled way of deriving online learning algorithms (and bounding their regret) from a minimax analysis. Of course, many online learning algorithms have been developed in various settings over the past few decades, but a comprehensive and systematic treatment was still lacking prior to \cite{RakhlinRL}. Starting from a general formulation of online learning as a (stateless) repeated game between a learner and an adversary, Rakhlin et al. \cite{RakhlinRL} analyze the minimax regret (value) of this online learning game, which is the regret (relative to a fixed competing strategy) that would be achieved if both the learner and the adversary play optimally. It was known before the work of \cite{RakhlinOR} that one could derive sublinear upper bounds on the minimax value in a nonconstructive manner. However, algorithm design was done on a case-by-case basis, and custom analysis techniques were needed in each case to derive performance guarantees matching these upper bounds. The work of \cite{RakhlinRL} bridges this gap between minimax value analysis and algorithm design: They have shown that, by choosing appropriate relaxations of a certain recursive decomposition of the minimax value, one can recover many known online learning algorithms and give a general recipe for developing new ones. In short, the framework proposed by \cite{RakhlinRL} can be used to convert an upper bound on the value of the game into an algorithm. 

Our main contribution is an extension of the framework of \cite{RakhlinRL} to online MDPs. Since online learning problems are studied in a state-free setting, it is not straightforward to generalize the ideas of \cite{RakhlinRL} to the case when the system has a state, and the technical nature of the arguments involved in online MDPs is significantly heavier than their state-free counterpart. We formulate the online MDP problem as a two-player repeated game with state variables and study its minimax value. We introduce the notion of an online MDP {\em relaxation} and show how it can be used to recover existing methods and to construct new algorithms. More specifically, we present two distinct approaches of moving from the original dynamic setting, where the state evolves according to a controlled Markov chain, to simpler static settings and constructing corresponding relaxations. The first approach uses Poisson inequalities for MDPs \cite{MeynTweedie} to reformulate the original dynamic setting as a static setting, where each possible state is associated with a separate online learning algorithm. We show that the algorithm proposed by \cite{EvenDar} arises from a particular relaxation, and we also derive a new algorithm in the spirit of \cite{Yu} which exhibits improved regret bounds. The second approach moves from the dynamic setting to a static setting by reducing the online MDP problem to an online linear optimization problem. After the reduction, we can directly capitalize on the framework of \cite{RakhlinRL}. We then derive a novel Online Mirror Descent (OMD) algorithm in the spirit of \cite{Zimin, DickCsaba} under a carefully designed relaxation over a certain convex set. In short, while the existing methods fall into two major categories, they both can be captured by the above two approaches, and these two approaches arise from the same general idea: move from the original dynamic setting to a static setting, derive the corresponding relaxation, and convert the relaxation into an algorithm.

The remainder of the paper is organized as follows. We close this section with a brief summary of our results and frequently used notation. Section~\ref{sec:setup} contains precise formulation of the online MDP problem and points out the general idea and major challenges. Section~\ref{sec:gfw} describes our proposed framework and contains the main result. The general framework includes two different methods of recovering and deriving algorithms. Section~\ref{sec:derivealgo} uses the first method and shows the power of our framework by recovering an existing method proposed in \cite{EvenDar} and further derives a new algorithm. Section~\ref{sec:derivedick} uses the second approach to derive a novel online MDP algorithm. Section~\ref{sec:clc} contains discussion about future research. Proofs of all intermediate results are relegated to the Appendix. 

\subsection{Summary of contributions}

We start by recasting an MDP with arbitrary costs as a one-sided {\em stochastic game}, where an agent who wishes to minimize his long-term average cost is facing a Markovian environment, which is also affected by arbitrary actions of an opponent. A stochastic game \cite{Shapley, Sorin} is a repeated two-player game, where the state changes at every time step according to a transition law depending on the current state and the moves of both players. Here we are considering a special type of a stochastic game, where the agent controls the state transition alone and the opponent chooses the cost functions. By ``one-sided'', we mean that the utility of the opponent is left unspecified. In other words, we do not need to study the strategy and objectives of the opponent, and only assume that the changes in the environment in response to the opponent's moves occur arbitrarily. As a result, we simply model the opponent as the environment.

A popular and common objective in such settings is regret minimization. Regret is defined as the difference between the cost the agent actually incurred, and what could have been incurred if the agent knew the observed sequence of cost functions in advance. We will give the precise definition of this regret notion in Section~\ref{sec:setup}. We start by studying the minimax regret, i.e., the regret the agent will suffer when both the agent and the environment play optimally. By applying the theory of dynamic programming for stochastic games \cite{Sorin}, we can give the strategy for the agent that achieves minimax regret (called the minimax strategy). It can be interpreted as choosing the best action that takes into account the current cost and the worst case future. Unfortunately, this minimax strategy in general is not computationally feasible due to the fact that the number of possible futures grows exponential with time. The idea is to find a way to approximate the term that represents the ``future'' and derive near-optimal strategy that is easy to compute using the approximation. 

Our main contribution is a construction of a general procedure for deriving algorithms in the online MDP setting. More specifically:
\begin{enumerate}
	\item Just as in the state-free setting considered by \cite{RakhlinRL}, we argue that algorithms can be constructed systematically by first deriving a sequence of upper bounds (relaxations) on the minimax value of the game, and then choosing actions which minimize these upper bounds.
	\item Once a relaxation and an algorithm are derived in this way, we give a general regret bound of that algorithm as follows:
	\begin{align*}
		\text{Expected regret} \le \text{Relaxation} + \text{Stationarization error}.
	\end{align*}
	The first term on the right-hand side of the above inequality is the expected relaxation, while the second term is an approximation error that results from approximating the Markovian evolution of the underlying process by a simpler stationary process using a procedure we refer to as {\em stationarization}. The first term can be analyzed using essentially the same techniques as the ones employed by \cite{RakhlinRL}, with some modifications; by contrast, the second term can be handled using only a novel combination of Markov chain methods. This approach significantly alleviates the technical burden of proving a regret bound as in the literature before our work.
	 \item Using the above procedure, we recover an existing method proposed by Even-Dar et al.\ in \cite{EvenDar}, which achieves $O(\sqrt{T})$ expected regret against the best stationary policy. We show that our derived relaxation gives us the same exponentially weighted average forecaster as in \cite{EvenDar} and leads to the same regret bound. We also derive a new algorithm using our proposed framework and argue that, while this new algorithm is similar in nature to the work of Yu et al.~\cite{Yu}, it has several advantages --- in particular, better scaling of the regret with the horizon $T$. Both of these algorithms are based on introducing a sequence of appropriately defined relative value functions, and thus can be viewed as instantiations of the first approach to online MDPs --- namely, the one rooted in dynamic programming.
	 \item We also present a different technique for deriving relaxations that implements the second approach to online MDPs --- the one rooted in the LP method. This approach allows us to reduce the online MDP problem to an online linear optimization problem over the state-action polytope. This reduction enables us to use the framework of \cite{RakhlinRL}, and the resulting relaxation leads to a novel OMD algorithm that is similar in spirit to the work of Dick et al.~\cite{DickCsaba}.
\end{enumerate}	

\subsection{Notation}
\label{ssec:notation}
\sloppypar We will denote the underlying finite state space and action space by $\sX$ and $\sU$, respectively. The set of all probability distributions on $\sX$ will be denoted by $\cP(\sX)$, and the same goes for $\sU$ and $\cP(\sU)$. A matrix $P = [P(u|x)]_{x \in \sX, u \in \sU}$ with nonnegative entries, and with the rows and the columns indexed by the elements of $\sX$ and $\sU$ respectively, is called {\em Markov} (or {\em stochastic}) if its rows sum to one: $\sum_{u \in \sU} P(u|x) = 1, \forall x \in \sX$. We will denote the set of all such Markov matrices (or randomized state feedback laws) by $\cM(\sU|\sX)$. Markov matrices in $\cM(\sU|\sX)$ transform probability distributions on $\sX$ into probability distributions on $\sU$: for any $\mu \in \cP(\sX)$ and any $P \in \cM(\sU|\sX)$, we have
\begin{align*}
	\mu P(u) \deq \sum_{x \in \sX} \mu(x)P(u|x), \qquad \forall u \in \sU.
\end{align*}
The same applies to Markov matrices on $\sX$ and to their action on the elements of $\cP(\sX)$.

The fixed and known stochastic transition kernel of the MDP will be denoted throughout by $K$ -- that is, $K(y|x,u)$ is the probability that the next state is $y$ if the current state is $x$ and the action $u$ is taken. For any Markov matrix (randomized state feedback law) $P \in \cM(\sU|\sX)$, we will denote by $K(y|x,P)$ the Markov kernel
\begin{align*}
	K(y|x,P) \deq \sum_{u \in \sU} K(y|x,u)P(u|x).
\end{align*}
Similarly, for any $\nu \in \cP(\sU)$,
\begin{align*}
	K(y|x,\nu) \deq \sum_{u \in \sU} K(y|x,u)\nu(u)
\end{align*}
(this can be viewed as a special case of the previous definition if we interpret $\nu$ as a state feedback law that ignores the state and draws a random action according to $\nu$). For any $\mu \in \cP(\sX)$ and $P \in \cM(\sU|\sX)$, $\mu \otimes P$ denotes the induced joint state-action distribution on $\sX \times \sU$:
$$
\mu \otimes P(x,u) = \mu(x)P(u|x), \qquad \forall (x,u) \in \sX \times \sU.
$$
We say that $P$ is {\em unichain} \cite{LermaLasserreMC} if the corresponding Markov chain with transition kernel $K(\cdot|\cdot,P)$ has a single recurrent class of states (plus a possibly empty transient class). This is equivalent to the induced kernel $K(\cdot|P)$ having a unique invariant distribution $\pi_P$ \cite{Seneta}.

The total variation (or $L_1$) distance between $\nu_1, \nu_2 \in \cP(\sU)$ is
\begin{align*}
\| \nu_1 - \nu_2 \|_1 \deq \sum_{u \in \sU} \lvert \nu_1(u) - \nu_2(u) \rvert.
\end{align*}
It admits the following variational representation:
\begin{align}\label{eq:TV_var}
	\| \nu_1 - \nu_2 \|_1 = \sup_{f:\, \| f \|_\infty \le 1} \left| \ave{\nu_1,f} - \ave{\nu_2,f} \right|,
\end{align}
where the supremum is over all functions $f : \sU \to \Reals$ with absolute value bounded by $1$, and we are using the linear functional notation for expectations:
\begin{align*}
	\ave{\nu,f} = \E_\nu[f] = \sum_{u \in \sU}\nu(u)f(u).
\end{align*}
The {\em Kullback--Leibler divergence} (or {\em relative entropy}) between $\nu_1$ and $\nu_2$ \cite{CoverThomas} is
\begin{displaymath}
D(\nu_1 \| \nu_2) \deq  \begin{cases}
\displaystyle\sum_{u \in \sU} \nu_1(u) \log \dfrac{\nu_1(u)}{\nu_2(u)} & \textrm{if ${\rm supp}(\nu_1) \subseteq {\rm supp}(\nu_2$)} \\
+ \infty & \textrm{otherwise}
\end{cases}
\end{displaymath}
where ${\rm supp}(\nu) \deq \{ u \in \sU: \nu(u) > 0 \}$ is the {\em support} of $\nu$. Here and in the sequel, we work with natural logarithms. The same applies, {\em mutatis mutandis}, to probability distributions on $\sX$.

We will also be dealing with binary trees that arise in symmetrization arguments, as in \cite{RakhlinRL}: Let $\cH$ be an arbitrary set. An $\cH$-valued tree $\mathbf h$ of depth $d$ is defined as a sequence $(\mathbf h_1, \ldots, \mathbf h_d)$ of mappings $\mathbf h_t: \{ \pm 1\}^{t-1} \to \cH$ for $t = 1, 2, \ldots, d$. Given a tuple $\eps = (\eps_1,\ldots,\eps_d) \in \{\pm 1\}^d$, we will often write $\mathbf h_t(\eps)$ instead of $\mathbf h_t(\eps_{1:t-1})$.

\section{Problem formulation}
\label{sec:setup}

We consider an online MDP with finite state and action spaces $\sX$ and $\sU$ and transition kernel $K(y|x,u)$. Let $\cF$ be a fixed class of functions $f : \sX \times \sU \to \Reals$, and let $x \in \sX$ be a fixed initial state. Consider an agent performing a controlled random walk on $\sX$ in response to signals coming from the environment. The agent is using mixed strategies to choose actions, where a mixed strategy is a probability distribution over the action space. The interaction between the agent and the environment proceeds as follows:
\begin{center}
\begin{tabular}{|l|}
\hline
$X_1=x$\\
for $t = 1,2,\ldots,T$\\
\ \ The agent observes the state $X_t$, selects a mixed strategy $P_t \in \cP(\sU)$, and then \\
draws an action $U_t$ from $P_t$\\
\ \ The environment simultaneously selects $f_t \in \cF$ and announces it to the agent\\
\ \ The agent incurs one-step cost $f_t(X_t, U_t)$\\
\ \ The system transitions to the next state $X_{t+1} \sim K(\cdot|X_t, U_t)$\\
end for\\
\hline 
\end{tabular}
\end{center}
Here, $T$ is a fixed finite horizon. We assume throughout that the environment is {\em oblivious} (or {\em open-loop}), in the sense that the evolution of the sequence $\{f_t\}$ is not affected by the state and action sequences $\{X_t\}$ and $\{U_t\}$. We view the above process as a two-player repeated game between the agent and the environment. At each $t \ge 1$, the process is at state $X_t = x_t$. The agent observes the current state $x_t$ and selects the mixed strategy $P_t$, where $P_t(u | x_t) = \Pr\{U_{t}=u|X_t=x_t\}$, based on his knowledge of all the previous states and current state $x^t = (x_1,\ldots,x_t)$ and the previous moves of the environment $f^{t-1} = (f_1,\ldots,f_{t-1})$. After drawing the action $U_t$ from $P_t$, the agent incurs the one-step cost $f_{t}(X_{t}, U_t)$. Adopting game-theoretic terminology \cite{BasarOlsder}, we define the agent's closed-loop {\em behavioral strategy} as a tuple $\bd{\gamma} = (\gamma_1,\ldots,\gamma_T)$, where $\gamma_t : \sX^t \times \cF^{t-1} \to \cP(\sU)$. Similarly, the environment's open-loop behavioral strategy is a tuple $\bd{f} = (f_1,\ldots,f_t)$. Once the initial state $X_1=x$ and the strategy pair $(\bd{\gamma},\bd{f})$ are specified, the joint distribution of the state-action process $(X^T,U^T)$ is well-defined.

Let $\cM_0 = \cM_0(\sU|\sX) \subseteq \cM(\sU|\sX)$ denote the subset of all Markov policies $P$, for which the induced state transition kernel $K(\cdot|\cdot,P)$ has a unique invariant distribution $\pi_P \in \cP(\sX)$. The goal of the agent is to minimize the expected {\em steady-state regret}
\begin{align} \label{eq:ssregretdef}
R^{\bd{\gamma},\bd{f}}_x \deq \E^{\bd{\gamma},\bd{f}}_x\left\{\sum^T_{t=1}f_t(X_t,U_t) - \inf_{P \in \cM_0}\, \E\left[\sum^T_{t=1}f_t(X,U)\right]\right\},
\end{align}
where the outer expectation $\E^{\bd{\gamma},\bd{f}}_x$ is taken w.r.t.\ both the Markov chain induced by the agent's behavioral strategy $\bd{\gamma}$ (including randomization of the agent's actions), the environment's behavior strategy $\bd{f}$, and the initial state $X_1 = x$. The inner expectation (after the infimum) is w.r.t.\ the state-action distribution $\pi_P \otimes P(x,u) = \pi_P(x)P(u|x)$, where $\pi_P$ denotes the unique invariant distribution of $K(\cdot|\cdot,P)$. The regret $R^{\bd{\gamma},\bd{f}}_x$ can be interpreted as the gap between the expected cumulative cost of the agent using strategy $\bd{\gamma}$ and the best steady-state cost the agent could have achieved in hindsight by using the best stationary policy $P \in \cM_0$ (with full knowledge of $\bd{f} = f^T$). This gap arises through the agent's lack of prior knowledge on the sequence of cost functions. 

Here we consider the steady-state regret, so that the expectation w.r.t.\ the state evolution in the comparator term $ \E\left[\sum^T_{t=1}f_t(X,U)\right]$ is taken over the invariant distribution $\pi_P$ instead of the Markov transition law $K(\cdot|\cdot,P)$ induced by $P$. Under the additional assumptions that  the cost functions $f_t$ are uniformly bounded and the induced Markov chains $K(\cdot|\cdot,P)$ are uniformly exponentially mixing for all $P \in \cM(\sU|\sX)$, the difference we introduce here by considering the steady state is bounded by a constant independent of $T$ \cite{EvenDar,Yu}, and so is negligible in the long run. In our main results, we only consider baseline policies in $\cM_0$ that are uniformly exponentially mixing, so we restrict our attention to the steady-state regret without any loss of generality.

\subsection{Minimax regret}
\label{ssec:minimaxanalysis}
We start our analysis by studying the value of the game (the minimax regret), which we first write down in {\em strategic form} as
\begin{align} \label{eq:strategy1}
	V(x) \deq \inf_{\bd{\gamma}}\sup_{\bd{f}}\, R^{\bd{\gamma},\bd{f}}_x = \inf_{\bd{\gamma}}\sup_{\bd{f}}\, \E^{\bd{\gamma},\bd{f}}_x\left[\sum^T_{t=1}f_t(X_t,U_t) - \Psi(\bd{f})\right],
\end{align}
where we have introduced the shorthand $\Psi$ for the comparator term:
\begin{align*}
	\Psi(\bd{f}) \deq \inf_{P \in \cM_0} \E\left[\sum^T_{t=1}f_t(X,U)\right]. 
\end{align*}
In operational terms, $V(x)$ gives the best value of the regret the agent can secure by any closed-loop behavioral strategy against the worst-case choice of an open-loop behavioral strategy of the environment. However, the strategic form of the value hides the {\em timing protocol} of the game, which encodes the information available to the agent at each time step. To that end, we give the following equivalent expression of $V(x)$ in {\em extensive form}:
\begin{proposition}\label{pps:strategyextensive} The minimax regret \eqref{eq:strategy1} is given by
	\begin{align} \label{eq:extensive2}
		V(x) = \inf_{P_1}\sup_{f_1}\ldots\inf_{P_T}\sup_{f_T} \E\left[\sum^T_{t=1}f_t(X_t,U_t)-\Psi(\bd{f})\right].
	\end{align}
\end{proposition}
\begin{proof} See Appendix~\ref{app:strategyextensive}.
\end{proof}
From this minimax formulation, we can immediately get an optimal algorithm that attains the minimax regret. To see this, we give an equivalent recursive form for the value of the game. For any $t \in \{0,1,\ldots,T-1\}$, any given prefix $f^{t} = (f_1,\ldots,f_{t})$ (where we let $f^0$ be the empty tuple ${\mathsf e}$), and any state $X_{t+1} = x$, define the conditional value
\begin{subequations}\label{eq:condvalue11}
\begin{align}
	{V}_t(x,f^t) &\deq \inf_{\nu \in \cP(\sU)} \sup_{f} \left\{ \sum_{u \in \sU}f(x,u)\nu(u) + \E\Big[{V}_{t+1}(Y,f_1,\ldots,f_t,f) \Big| x,\nu\Big]\right\}, \qquad t = T-1,\ldots,0 \\
	{V}_T(x,f^T) &\deq - \Psi(\bd{f}).
\end{align}
\end{subequations}
\begin{remark} {\em Recursive decompositions of this sort arise frequently in problems involving decision-making in the presence of uncertainty. For instance, we may view \eqref{eq:condvalue11} as a dynamic program for a finite-horizon minimax control problem \cite{BertsekasRhodes}. Alternatively, we can think of \eqref{eq:condvalue11} as applying the {\em Shapley operator} \cite{Sorin} to the conditional value in a two-player stochastic game, where one player controls only the state transitions, while the other player specifies the cost function. A promising direction for future work is to derive some characteristics of the conditional value from analytical properties of the Shapley operator.}
\end{remark}
From Proposition~\ref{pps:strategyextensive}, we see that $V(x) = V_0(x,\mathsf{e})$. Moreover, we can immediately write down the minimax-optimal behavioral strategy for the agent:
\begin{align*}
	{\gamma}_{t+1}(x,f^{t}) &= \argmin_{\nu \in \cP(\sU)}\sup_{f \in \cF}\left\{ \sum_{u \in \sU}f(x,u)\nu(u) + \E\Big[{V}_{t+1}(Y,f_1,\ldots,f_{t},f) \Big| x,\nu \Big]\right\}, \qquad t = 0,\ldots,T-1.
\end{align*}
Note that the expression being minimized is a supremum of affine functions of $\nu$, so it is a lower-semicontinuous function of $\nu$. Any lower-semicontinuous function achieves its infimum on a compact set. Since the probability simplex $\cP(\sU)$ is compact, we are assured that a minimizing $\nu$ always exists. Using the above strategy at each time step, we can secure the minimax regret in the worst-case scenario. Note also that this strategy is very intuitive: it balances the tendency to minimize the present cost against the risk of incurring high future costs.  However, with all the future infimum and supremum pairs involved, computing this conditional value is intractable. As a result, the minimax optimal strategy is not computationally feasible. The idea is to give tight bounds of the conditional value, which can be minimized to form a near-optimal strategy. We address this challenge by developing computable bounds for the conditional value functions, choosing a strategy based on these bounds. In general, tighter bounds yield lower regret and looser bounds are easier to compute, and various online MDP methods occupy different points in this domain.

In the spirit of \cite{RakhlinRL}, we come up with approximations of the conditional value ${V}_t(x,f^t)$ in \eqref{eq:condvalue11}. We say that a sequence of functions $\wh{V}_t : \sX \times \cF^t \to \Reals$ is an {\em admissible relaxation} if
\begin{subequations}\label{eq:condtionvalueshapley}
\begin{align} 
	\wh{V}_t(x,f^t) &\ge \inf_{\nu \in \cP(\sU)} \sup_{f} \left\{ \sum_{u \in \sU}f(x,u)\nu(u) + \E[\wh{V}_{t+1}(Y,f_1,\ldots,f_t,f)|x,\nu]\right\}, \qquad t = T-1,\ldots,0 \\
	\wh{V}_T(x,f^T) &\ge - \Psi(\bd{f}).
\end{align}
\end{subequations}
We can associate a behavioral strategy $\wh{\bd{\gamma}}$ to any admissible relaxation as follows:
\begin{align*}
	\wh{\gamma}_t(x,f^{t-1}) &= \argmin_{\nu \in \cP(\sU)}\sup_{f \in \cF}\left\{ \sum_{u \in \sU}f(x,u)\nu(u) + \E\Big[\wh{V}_{t}(Y,f_1,\ldots,f_{t-1},f) \Big| x,\nu \Big]\right\}.
\end{align*}
\begin{proposition} \label{pps:admissbound1} Given an admissible relaxation $\{\wh{V}_t\}^T_{t=0}$ and the associated behavioral strategy $\wh{\bd{\gamma}}$, for any open-loop strategy of the environment we have the regret bound
	\begin{align*}
		R^{\wh{\bd{\gamma}},\bd{f}}_x = \E^{\wh{\bd{\gamma}},\bd{f}}_x\left[ \sum^T_{t=1}f_t(X_t,U_t)-\Psi(\bd{f})\right]    \le \wh{V}_0(x).
	\end{align*}
\end{proposition}
\begin{proof} See Appendix~\ref{app:admissbound1}.\end{proof}
Based on the above sequential decompositions, it suffices to restrict attention only to Markov strategies for the agent, i.e., sequences of mappings $\gamma_t : \sX \times \cF^{t-1} \to \cP(\sU)$ for all $t$, so that $U_t$ is conditionally independent of $X^{t-1},U^{t-1}$ given $X_t,f^{t-1}$. From now on, we will just say ``behavioral strategy'' and really mean ``Markov behavioral strategy.'' In other words, given $X_t,f^{t-1}$, the history of past states and actions is {\em irrelevant}, as far as the value of the game is concerned.
\begin{remark} {\em What happens if the environment is nonoblivious? Yu et al.~\cite{Yu} gave a simple counterexample of an aperiodic and recurrent MDP to show that the regret is linear in $T$  regardless of the agent's policy when the opponent can adapt to the agent's state trajectory. We can gain additional insight into the challenges associated with an adaptive environment from the perspective of the minimax regret. In particular, an adaptive environment's {\em closed-loop}
behavioral strategy is $\bd{\delta} = (\delta_1,\ldots,\delta_T)$ with $\delta_t : \sX^t \times \sU^{t-1} \to \cP(\cF)$, and the corresponding regret will be given by
	\begin{align*}
		\E^{\bd{\gamma},\bd{\delta}}_x\left[ \sum^T_{t=1}f_t(X_t,U_t) - \Psi(\bd{f})\right] &\le \E^{\bd{\gamma},\bd{\delta}}_x\left[ \sum^T_{t=1}f_t(X_t,U_t) + \wh{V}_T(X_{T+1},f^T)\right] \\
		&= \E^{\bd{\gamma},\bd{\delta}}_x\left[ \sum^{T-1}_{t=1}f_t(X_t,U_t) \right] + \E^{\bd{\gamma},\bd{\delta}}_x\left[ f_T(X_T,U_T)+ \wh{V}_T(X_{T+1},f^T)\right].
	\end{align*}
	Let's analyze the last two terms:
	\begin{align*}
&	\E^{\bd{\gamma},\bd{\delta}}_x\left[ f_T(X_T,U_T)+ V_T(X_{T+1},f^T)\right]\nonumber\\
 &= \int_{\sX^T,\cF^T} {\mathbb P}({\rm d}x^T,{\rm d}\bd{f})\int_\sU P({\rm d}u_T|x_T,f^{T-1})\left\{f_T(x_T,u_T) + \E\Big[\wh{V}_T\big(X_{T+1},f^T\big)\Big|x^T,f^T\Big]\right\}.
\end{align*}
In the above conditional expectation, $\bd{f}$ may depend on the entire $x^T$, so we cannot replace this conditional expectation by $\E[\cdot|x_T,\gamma_T(x_T)]$. This implies we cannot get similar results as in Proposition~\ref{pps:admissbound1} in a fully adaptive environment.}
\end{remark}

\subsection{Major challenges}
From Proposition~\ref{pps:admissbound1}, we can see that we can bound the expected steady-state regret in terms of the chosen relaxation. Ideally, if we construct an admissible relaxation by deriving certain upper bounds on the conditional value and implement the associated behavioral strategy, we will obtain an algorithm that achieves the regret bound corresponding to the relaxation. In principle, this gives us a general framework to develop low-regret algorithms for online MDPs. However, with an additional state variable involved, it is difficult to derive admissible relaxations $\wh{V}_t(x,f^t)$ to bound the conditional value. The difficulty stems from the fact that now the current cost depends not only  on the current action, but also on past actions. Our plan is to reduce this setting to a simpler setting where there is no Markov dynamics involved. In that setting, we will be able to capitalize on the ideas of \cite{RakhlinOR, RakhlinRL} in two different ways.  More specifically, using Rademacher complexity tools introduced by \cite{RakhlinOR, RakhlinRL}, we can derive algorithms in simpler static settings and then transfer them to the original problem. In the same vein, we will also prove a general regret bound for the derived algorithms. Thus we will have a general recipe for developing algorithms and showing performance guarantees for online MDPs.

\section{The general framework for constructing algorithms in online MDPs}
\label{sec:gfw}

As mentioned in the above section, the main challenge to overcome is the dependence of the conditional value in \eqref{eq:condtionvalueshapley} on the state variable. Our plan is to reduce the original online MDP problem to a simpler one, where there is no Markov dynamics.

We proceed with our plan in several steps. First, we introduce a stationarization technique that will allow us to reduce the online MDP setting to a simpler setting without Markov dynamics. This effectively decouples current costs from past actions. Note that this reduction is fundamentally different from just naively applying stateless online learning methods in an online MDP setting, which would amount to a very poor stationarization strategy with larger errors and consequently large regret bounds. In contrast, our proposed stationarization performs the decoupling with minimal loss in accuracy by exploiting the transition kernel, yielding lower regret bounds. Using the stationarization idea, we present two different approaches to construct relaxations, aiming to recover and derive two distinct lines of existing methods. We call the first approach the {\em value-function approach}. Making use of Poisson inequalities for MDPs \cite{MeynTweedie}, we state a new admissibility condition for relaxations that differs from the admissibility condition in \eqref{eq:condtionvalueshapley} in that there is no conditioning on the state variable. The advantage of working with this new type of relaxation is that the corresponding admissibility conditions are much easier to verify. The second approach is called the {\em convex-analytic approach}. By treating the online MDP problem as an online linear optimization problem, we are able to adopt the idea of \cite{RakhlinRL} in a more straightforward way, and use the admissibility condition in \cite{RakhlinRL} to construct relaxations and derive corresponding algorithms. These two approaches can recover different categories of existing methods (it should be pointed out, however, that there is a natural equivalence between these two approaches: the relative-value function arises as a Lagrange multiplier associated with the invariance constraint that defines the state-action polytope; cf.~\cite[Sec.~9.2]{MeynCTCN} for details). The main result of this section is that we can apply any algorithm derived in the simpler static setting to the original dynamic setting and automatically bound its regret.

\subsection{Stationarization} 
\label{ssec:stationarization}

As before, we let $K$ denote the fixed and known transition law of the MDP. Following \cite{EvenDar} and \cite{Yu}, we assume the following ``uniform mixing condition'': There exists a finite constant $\tau > 0$ such that for all Markov policies $P \in \cM(\sU|\sX)$ and all distributions $\mu_1,\mu_2 \in \cP(\sX)$,
\begin{align} \label{eq:mixingtimeeq}
\| \mu_1 K(\cdot|P) - \mu_2 K(\cdot|P) \|_1  \le e^{-1/\tau} \| \mu_1 - \mu_2 \|_1,
\end{align}
where $K(\cdot|P) \in \cM(\sX|\sX)$ is the Markov matrix on the state space induced by $P$. In other words, the collection of all state transition laws induced by all Markov policies $P$ is {\em uniformly mixing}. Here we assume, without loss of generality, that $\tau \ge 1$. As pointed out in \cite{NeuCsaba}, this uniform mixing condition is actually stronger than the unichain assumption: $K(\cdot|\cdot,P)$ is unichain for any choice of $P \in \cM(\sU|\sX)$ --- see Section~\ref{ssec:notation} for definitions. The uniform mixing condition implies that the transition kernel of every policy is a scrambling matrix. (A matrix $K \in \cM(\sX|\sX)$ is scrambling if and only if for any pair $x, x' \in \sX$ there exists at least one $y \in \sX$, such that $y$ can be reached from both $x$ and $x'$ in one step with strictly positive probability using $K$ as transition matrix.) 

Consider now a behavioral strategy $\bd{\gamma} = (\gamma_1,\ldots,\gamma_T)$ for the agent. For a given choice $\bd{f} = (f_1,\ldots,f_T)$ of costs, the following objects are well-defined:
\begin{itemize}
	\item $P^{\bd{\gamma},\bd{f}}_t \in \cM(\sU|\sX)$ --- the Markov matrix that governs the conditional distribution of $U_t$ given $X_t$, i.e.,
	\begin{align*}
		P^{\bd{\gamma},\bd{f}}_t(u|x) = \left[\gamma_t(x,f^{t-1})\right](u);
	\end{align*}
	\item $\mu^{\bd{\gamma},\bd{f}}_t \in \cP(\sX)$ --- the distribution of $X_t$;
	\item $K^{\bd{\gamma},\bd{f}}_t \in \cM(\sX|\sX)$ --- the Markov matrix that describes the state transition from $X_t$ to $X_{t+1}$, i.e.,
	\begin{align*}
		K^{\bd{\gamma},\bd{f}}_t(y|x) = K(y|x,P^{\bd{\gamma},\bd{f}}_t) \equiv \sum_u K(y|x,u) P^{\bd{\gamma},\bd{f}}_t(u|x);
	\end{align*}
	\item $\pi^{\bd{\gamma},\bd{f}}_t \in \cP(\sX)$ --- the unique stationary distribution of $K^{\bd{\gamma},\bd{f}}_t$, satisfying $\pi^{\bd{\gamma},\bd{f}}_t = \pi^{\bd{\gamma},\bd{f}}_t K^{\bd{\gamma},\bd{f}}_t$, where existence and uniqueness are guaranteed by virtue of the unichain assumption;
	\item $\eta^{\bd{\gamma},\bd{f}}_t = \ave{\pi^{\bd{\gamma},\bd{f}}_t \otimes P^{\bd{\gamma},\bd{f}}_t, f_t}$ --- the steady-state cost at time $t$.
\end{itemize}
Moreover, for any other state feedback law $P \in \cM(\sU|\sX)$, we will denote by $\eta^{P,\bd{f}}_t$ the steady-state cost $\ave{\pi_P \otimes P,f_t}$, where $\pi_P$ is the unique invariant distribution of $K(\cdot|\cdot,P)$. 

It will be convenient to introduce the regret w.r.t.\ a fixed $P \in \cM(\sU|\sX)$ with initial state $X_1 = x$:
\begin{align*}
	R^{\bd{\gamma},\bd{f}}_x(P) &\deq \E^{\bd{\gamma},\bd{f}}_x \left[ \sum^T_{t=1} f_t(X_t,U_t) -  \sum^T_{t=1}\eta^{P,\bd{f}}_t\right] \\
	&= \sum^T_{t=1}\left[\ave{\mu^{\bd{\gamma},\bd{f}}_t \otimes P^{\bd{\gamma},\bd{f}}_t,f_t} - \ave{\pi_P \otimes P,f_t}\right],
\end{align*}
as well as the {\em stationarized regret}
\begin{align*}
	\bar{R}^{\bd{\gamma},\bd{f}}(P) &\deq \sum^T_{t=1}\left(\eta^{\bd{\gamma},\bd{f}}_t -\eta^{P,\bd{f}}_t\right)\\
	&= \sum^T_{t=1}\left[\ave{\pi^{\bd{\gamma},\bd{f}}_t \otimes P^{\bd{\gamma},\bd{f}}_t,f_t} - \ave{\pi_P \otimes P, f_t}\right].
\end{align*}
Using \eqref{eq:TV_var}, we get the bound
\begin{align}\label{eq:regret_two_terms}
	R^{\bd{\gamma},\bd{f}}_x(P) &\le \bar{R}^{\bd{\gamma},\bd{f}}(P) + \sum^T_{t=1} \| f_t \|_\infty \| \mu^{\bd{\gamma},\bd{f}}_t - \pi^{\bd{\gamma},\bd{f}}_t \|_1.
\end{align}
The key observation here is that the task of analyzing the regret $R^{\bd{\gamma},\bd{f}}_x(P)$ splits into separately upper-bounding the two terms on the right-hand side of \eqref{eq:regret_two_terms}:  the stationarized regret $\bar{R}^{\bd{\gamma},\bd{f}}(P)$ and the stationarization error $\sum^T_{t=1} \| f_t \|_\infty \| \mu^{\bd{\gamma},\bd{f}}_t - \pi^{\bd{\gamma},\bd{f}}_t \|_1$. The latter can be handled using Markov chain techniques. We now present two distinct approaches to tackle the former: the value-function approach and the convex-analytic approach.

\subsection{The value-function approach}
\label{ssec:value_function_approach}

The value-function approach relies on the availability of a so-called \textit{reverse Poisson inequality, which can be thought of as a generalization of the Poisson equation from the theory of MDPs \cite{MeynTweedie}.} Fix a Markov matrix $P \in \cM(\sU|\sX)$ and let $\pi_P \in \cP(\sX)$ be the (unique) invariant distribution of $K(\cdot|\cdot,P)$.  Then we say that $\wh{Q} : \sX \times \sU \to \Reals$ satisfies the reverse Poisson inequality with {\em forcing function} $g : \sX \times \sU \to \Reals$ if
	\begin{align}\label{eq:RPI}
		\E \Big[\wh{Q}(Y,P) \Big| x,u \Big] - \wh{Q}(x,u) &\ge - g(x,u) + \ave{\pi_P \otimes P, g}, \qquad \forall (x,u) \in \sX \times \sU
	\end{align}
	where
	\begin{align*}
		\wh{Q}(y,P) \deq \sum_{u \in \sU}P(u|y)\wh{Q}(y,u)
	\end{align*}
	and $\E[\cdot|x,u]$ is w.r.t.\ the transition law $K(y|x,u)$. We should think of this as a relaxation of the Poisson equation \cite{MeynTweedie}, i.e., when \eqref{eq:RPI} holds with equality. The Poisson equation  arises naturally in the theory of Markov chains and Markov decision processes, where it provides a way to evaluate the long-term average cost along the trajectory of a Markov process. We are using the term ``reverse Poisson inequality'' to distinguish \eqref{eq:RPI} from the Poisson inequality, which also arises in the theory of Markov chains and is obtained by replacing $\ge$ with $\le$ in \eqref{eq:RPI} \cite{MeynTweedie}.

Here we impose the following assumption that we use throughout the rest of the paper:
\begin{assumption} \label{as:existboundQ} For any $P \in \cM(\sU|\sX)$ and any $f \in \cF$, there exists some $\wh{Q}_{P,f} : \sX \times \sU \to \Reals$ that solves the reverse Poisson inequality for $P$ with forcing function $f$. Moreover,
	\begin{align*}
		L(\sX,\sU,\cF) \deq \sup_{P \in \cM(\sU|\sX)} \sup_{f \in \cF} \| \wh{Q}_{P,f} \|_\infty < \infty.
	\end{align*}
\end{assumption}
\begin{remark} {\em In Section~\ref{sec:derivealgo}, we will show this assumption is automatically satisfied when $K$ is a unichain model (or, more generally, when all stationary Markov policies are uniformly mixing, as in Eq.~\eqref{eq:mixingtimeeq}).}
\end{remark}

The main consequence of the reverse Poisson inequality is the following:
\begin{lemma}[Comparison principle]\label{lem:comparisonch3} Suppose that  $\wh{Q}$ satisfies the reverse Poisson inequality \eqref{eq:RPI} with forcing function $g$. Then for any other Markov matrix $P'$ we have
	\begin{align*}
		\ave{\pi_P \otimes P, g} - \ave{\pi_{P'} \otimes P', g} &\le \sum_x \pi_{P'}(x) \sum_u \left[P(u|x)\wh{Q}(x,u) - P'(u|x)\wh{Q}(x,u)\right]
	\end{align*}
\end{lemma}
\begin{proof} See Appendix~\ref{app:comparisonch3}. \end{proof}
Armed with this lemma, we can now analyze the stationarized regret $\bar{R}^{\bd{\gamma},\bd{f}}(P)$:  suppose that, for each $t$, $\wh{Q}^{\bd{\gamma},\bd{f}}_t$ satisfies reverse Poisson inequality for $P^{\bd{\gamma},\bd{f}}_t$ with forcing function $f_t$. Then we apply the comparison principle to get
\begin{align*}
	\eta^{\bd{\gamma},\bd{f}}_t - \eta^{P,\bd{f}}_t &\le \sum_{x}\pi_P(x) \left( \sum_u P^{\bd{\gamma},\bd{f}}_t(u|x)\wh{Q}^{\bd{\gamma},\bd{f}}_t(x,u) - P(u|x)\wh{Q}^{\bd{\gamma},\bd{f}}_t(x,u)\right).
\end{align*}
This in turn yields
\begin{align*}
	R^{\bd{\gamma},\bd{f}}_x(P) &\le \sum_x \pi_P(x) \sum^T_{t=1} \left( \sum_u P^{\bd{\gamma},\bd{f}}_t(u|x)\wh{Q}^{\bd{\gamma},\bd{f}}_t(x,u) - P(u|x)\wh{Q}^{\bd{\gamma},\bd{f}}_t(x,u)\right) \\ 
	&\qquad \qquad + \sum^T_{t=1} \| f_t \|_\infty \| \mu^{\bd{\gamma},\bd{f}}_t - \pi^{\bd{\gamma},\bd{f}}_t \|_1.
\end{align*}
Note that $\wh{Q}^{\bd{\gamma},\bd{f}}_t$ depends functionally on $P^{\bd{\gamma},\bd{f}}_t$ and on $f_t$, which in turn depend functionally on $f^t$ but not on $f_{t+1},\ldots,f_T$. This ensures that any algorithm using $\wh{Q}^{\bd{\gamma},\bd{f}}_t$  respects the causality constraint that any decision made at time $t$ depends only on information available by time $t$.

Focusing on stationarized regret and upper-bounding it in terms of the $\wh Q$-functions is one of the key steps that let us consider a simpler setting without Markov dynamics. The next step is to define a new type of relaxation with an accompanying new admissibility condition for this simpler setting. That is, we will find a relaxation and admissibility condition for the stationarized regret rather than for the expected steady-state regret directly. A new admissibility condition is needed because we have decoupled current costs from past actions, which makes the previous admissibility condition \eqref {eq:condtionvalueshapley} inapplicable. The new admissibility condition is similar to the one in \cite{RakhlinRL}, which was derived in a stateless setting. The difference is that we are still in a {\em state-dependent} setting in the sense that the new type of relaxation is indexed by the state variable. Now instead of having a Markov dynamics that depends on the state, we consider all the states in parallel and have a separate algorithm running on each state. The interaction between different states is generated by providing these algorithms with common information that comes from the actual dynamical process. Thus, starting from this new admissibility condition, we further construct algorithms using relaxations and then use Lemma~\ref{lem:comparisonch3} to bound the regret of these algorithms.

For each $x \in \sX$, let $\cH_x$ denote the class of all functions $h_x : \sU \to \Reals$ for which there exist some $P \in \cM(\sU|\sX)$ and $f \in \cF$, such that
\begin{align*}
	h_x(u) = \wh{Q}_{P,f}(x,u), \qquad \forall u \in \sU.
\end{align*}
We say that a sequence of functions $\wh{W}_{x,t} : \cH^t_x \to \Reals, t = 0,\ldots,T$, is an admissible relaxation at state $x$ if the following  condition holds for any $h_{x,1},\ldots,h_{x,T} \in \cH_x$:
\begin{subequations} \label{eq:generalrelaxation1}
\begin{align}
	\wh{W}_{x,T}(h^T_x) &\ge - \inf_{\nu \in \cP(\sU)} \E_{U \sim \nu}\left[\sum^T_{t=1} h_{x,t}(U)\right],  \\
	\wh{W}_{x,t}(h^t_x) &\ge \inf_{\nu \in \cP(\sU)} \sup_{h_x \in \cH_x} \left\{ \E_{U \sim \nu}[h_x(U)] + \wh{W}_{x,t+1}(h^t_x,h_x) \right\}, \qquad t = T-1,\ldots,0.
\end{align}
\end{subequations}
Given such an admissible relaxation, we can associate to it a behavioral strategy
\begin{align*}
	\wh{\gamma}_t(x,f^{t-1}) &= P^{\wh{\bd{\gamma}},\bd{f}}_t(\cdot|x) = \argmin_{\nu \in \cP(\sU)} \sup_{h_x \in \cH_x} \left\{ \E_{U \sim \nu} [h_x(U)] +  \wh{W}_{x,t}(h^{t-1}_x,h_x)\right\} \\
	h_{y,t} &= \wh{Q}^{\wh{\bd{\gamma}},\bd{f}}_t(y,\cdot), \quad \forall y \in \sX.
\end{align*}
(Even though the above notation suggests the dependence of $h_{y,t}$ on the $T$-tuples $\bd{\gamma}$ and $\bd{f}$, this dependence at time $t$ is only w.r.t.\ $\gamma^t$ and $f^t$, so the resulting strategy is still causal.)

The relaxation $\{ \wh{W}_{x,t}\}^T_{t=1}$ at state $x$ is a sequence of upper bounds on the conditional value of the online learning game associated with that state. In this game, at time step $t$, the agent chooses actions $u_t \in \sU$ and the environment chooses function $h_{x,t} \in \cH_x$. Although this relaxation is still state-dependent, there is no Markov dynamics involved here, which means that now the state-free techniques of \cite{RakhlinRL} can be brought to bear on the problem of constructing algorithms and bounding their regret. Specifically, we derive a separate relaxation $\{ \wh{W}_{x,t}\}^T_{t=1}$ and the associated behavioral strategy for each state $x \in \sX$. Then we assemble these into an overall algorithm for the MDP as follows: if at time $t$ the state $X_t = x$, the agent will choose actions according to the corresponding behavioral strategy $\wh{\gamma}_t(x,\cdot)$. Note that although the agent's behavioral strategy switches between different relaxations depending on the current state, the agent still needs to update all the $h$-functions simultaneously for all the states. This is because the computation of the $h$-functions (in terms of the $\wh Q$ functions) requires the knowledge of the behavioral strategy at other states. In other words, the algorithm has to keep updating all the relaxations in parallel for all states. 

Under the constructed relaxation, the value-function approach amounts to the following:
\begin{theorem} \label{thm:main_ch3} Suppose that the MDP is unichain, the environment is oblivious, and  Assumption~\ref{as:existboundQ} holds. Then, for any family of admissible relaxations given by \eqref{eq:generalrelaxation1} and the corresponding behavioral strategy \ $\wh{\bd{\gamma}}$, we have
	\begin{align} \label{eq:thm2eq1}
		R^{\wh{\bd{\gamma}},\bd{f}}_x=\E^{\wh{\bd{\gamma}},\bd{f}}_x\left[\sum^T_{t=1}f_t(X_t,U_t) - \Psi(\bd{f})\right] &\le \sup_{P \in \cM(\sU|\sX)}\sum_{x}\pi_P(x) \wh{W}_{x,0} + C_\cF  \sum^T_{t=1} \| \mu^{\wh{\bd{\gamma}},\bd{f}}_t - \pi^{\wh{\bd{\gamma}},\bd{f}}_t \|_1
	\end{align}
	where $C_\cF = \sup_{f \in \cF} \| f \|_\infty$. 
\end{theorem}
\begin{proof} See Appendix~\ref{app:main_ch3}. \end{proof}

This general framework gives us a recipe for deriving algorithms for online MDPs. First, we use stationarization to pass to a simpler setting without Markov dynamics. Here we need to find the $\wh Q_t$ functions satisfying \eqref{eq:RPI} with forcing function $f_t$ at each time $t$. In this simpler setting, we associate each state with a separate online learning game. Next, we derive appropriate relaxations (upper bounds on the conditional values) for each of these online learning games. Then we plug the relaxation into the admissibility condition \eqref{eq:generalrelaxation1} to derive the associated algorithm. This algorithm in turn gives us a behavioral strategy for the original online MDP problem, and Theorem~\ref{thm:main_ch3} automatically gives us a regret bound for this strategy. We emphasize that, in general, multiple different relaxations are possible for a given problem, allowing for a flexible tradeoff between computational costs and regret.

We have reduced the original problem to a collection of standard online learning problems, each of which is associated with a particular state. We proceed by constructing a separate relaxation for each of these problems. Because we have removed the Markov dynamics, we may now use available techniques for constructing these relaxations. In particular, as shown by \cite{RakhlinOR}, a particularly versatile method for constructing relaxations relies on the notion of {\em sequential Rademacher complexity} (SRC). 
 
\subsection{The convex-analytic approach}
\label{ssec: omd}
In the preceding section, we have developed a procedure for recovering and deriving policies for online MDPs using relative-value functions that arise from revere Poisson inequalities. Now we show a complementary procedure that allows us to use an admissible relaxation with no conditioning on state variables. Specifically, we reduce the online MDP problem to an online linear optimization problem through stationarization, and then directly use the framework of \cite{RakhlinRL} to derive a relaxation and an algorithm, which is similar in spirit to the algorithm proposed recently by Dick et al.~\cite{DickCsaba}. The idea behind this convex-analytic method is closely related to the well-known fact that the dynamic optimization problem for an MDP can be reformulated as a ``static'' linear optimization problem a certain polytope, and therefore can be solved using LP methods \cite{Manne,Borkar}. Under this reformulation, we are in a state-free setting in the sense that the relaxation is no longer indexed by the state variable, and the policy for the agent is computed from a certain joint distribution on states and actions via Bayes' rule.

As before, we start with the stationarization step. Recall that we decompose the regret $R^{\bd{\gamma},\bd{f}}_x(P)$ into two parts: the stationarized regret $\bar{R}^{\bd{\gamma},\bd{f}}(P)$ and the stationarization error $\sum^T_{t=1} \| f_t \|_\infty \| \mu^{\bd{\gamma},\bd{f}}_t - \pi^{\bd{\gamma},\bd{f}}_t \|_1$, that is:
\begin{align} 
	R^{\bd{\gamma},\bd{f}}_x(P)  & \le \bar{R}^{\bd{\gamma},\bd{f}}(P) + \sum^T_{t=1} \| f_t \|_\infty \| \mu^{\bd{\gamma},\bd{f}}_t - \pi^{\bd{\gamma},\bd{f}}_t \|_1 \nonumber \\
	&=  \sum^T_{t=1}\left(\eta^{\bd{\gamma},\bd{f}}_t -\eta^{P,\bd{f}}_t\right) + \sum^T_{t=1} \| f_t \|_\infty \| \mu^{\bd{\gamma},\bd{f}}_t - \pi^{\bd{\gamma},\bd{f}}_t \|_1 \nonumber\\
	&= \sum^T_{t=1}\left[\ave{\pi^{\bd{\gamma},\bd{f}}_t \otimes P^{\bd{\gamma},\bd{f}}_t,f_t} - \ave{\pi_P \otimes P, f_t}\right] + \sum^T_{t=1} \| f_t \|_\infty \| \mu^{\bd{\gamma},\bd{f}}_t - \pi^{\bd{\gamma},\bd{f}}_t \|_1. \label{eq:regretdecomp}
\end{align}
Now, let $\mathcal G \subset \cP(\sX \times \sU)$ denote the set of all \textit{ergodic occupation measures
\begin{align}\label{eq:SAP}
\mathcal G \deq \Big\{ \nu \in \cP(\sX \times \sU): \sum_{x,u} K(y|x,u) \nu(x,u) = \sum_u \nu(y,u), \forall y \in \sX \Big\}.
\end{align}
The set $\mathcal G$ is convex, and is defined by a finite collection of linear equality and inequality constraints. Hence, it is a convex polytope in $\Reals^{|\sX \times \sU|}$ (in fact, it is often referred to as the \textit{state-action polytope} of the MDP \cite{Puterman}). Every element in $\mathcal G$ can be decomposed in the form
$$
\nu(x,u) = \pi_P(x) \otimes P(u |x), \quad x \in \sX, u \in \sU
$$
for some randomized Markov policy $P \in \cM(\sU|\sX)$, where $\pi_P$ is the invariant distribution of the Markov kernel 
\begin{align*}
	K_P(x'|x) \deq \sum_{u \in \sU} K(x'|x,u)P(u|x), \qquad \forall x,x' \in \sX
\end{align*}
induced by $P$.  For this reason, the linear equality and inequality constraints that define $\mathcal G$ are also called the \textit{invariance constraints}. Conversely, any element $\nu \in \mathcal G$ induces a Markov policy
\begin{align}\label{eq:from_ergodic_to_policy}
	P_\nu(u|x) \deq \frac{\nu(x,u)}{\sum_{v \in \sU}\nu(x,v)}, \qquad \forall u \in \sU(x)
\end{align}
where $\sU(x)$ is the set of all states for which the denominator of \eqref{eq:from_ergodic_to_policy} is nonzero. }

With the definition of the set $\mathcal G$ at hand, now it is easy to see that the first term of \eqref{eq:regretdecomp} is the regret of an online \textit{linear} optimization problem, where, at each time step $t$, the agent is choosing an occupation measure $\nu_t = \pi^{\bd{\gamma},\bd{f}}_t \otimes P^{\bd{\gamma},\bd{f}}_t$ from the set $\mathcal G$ (here we omit the dependence of $\nu_t$ on $\bd{\gamma},\bd{f}$ for simplicity), and the environment is choosing the one-step cost function $f_t$. The one-step linear cost function incurred by the agent is $\ave{\nu_t,f_t}$. Since we can recover a policy from an occupation measure, we just need to find a slowly changing sequence of occupation measures, to ensure simultaneously that the first term of \eqref{eq:regretdecomp} and the stationarization error are both small. Now we have mapped an online MDP problem to an online linear optimization problem. As we mentioned earlier, the idea behind this mapping is simply the fact that average-cost optimal control problem can be cast as a LP over the state-action polytope \cite{Manne,Borkar}. 

For reasons that will become apparent later, it is convenient to consider regret with respect to policies induced by elements of a given subset ${\cal G}'$ of ${\cal G}$. With that in mind, let us denote by $\cM({\cal G}')$ the set of all policies $P \in \cM(\sU|\sX)$ that have the form $P_\nu$ for some $\nu \in {\cal G}'$. For the resulting online linear optimization problem, we can immediately apply the framework of \cite{RakhlinRL} to derive novel relaxations and online MDP algorithms. For any $t \in \{0,1,\ldots,T-1\}$, any given prefix $f^{t} = (f_1,\ldots,f_{t})$, define the conditional value recursively via
\begin{align}
V_T({\mathcal G}' | f_1, \ldots,f_{t}) = \inf_{\nu \in {\mathcal G}'} \sup_{f \in \cF} \left\{ \ave{\nu,f} + V_T({\mathcal G}' | f_1, \ldots,f_{t},f) \right\},
\end{align}
where $V_T({\mathcal G}' | f_1, \ldots,f_{T}) = -\inf_{\nu \in {\mathcal G}'} \sum^T_{t=1} \ave{\nu,f_t}$, and $V_T({\mathcal G}') \equiv V_T({{\mathcal G}'} | {\mathsf e})$ is the minimax regret of the game. Note that we are explicitly indicating the fact that the optimization takes place over the ergodic occupation measures in ${{\mathcal G}'}$.
The minimax optimal algorithm specifying the mixed strategy of the player can be written as
 \begin{align}
\nu_t = \argmin_{\nu \in {\mathcal G}'} \sup_{f \in \cF}\left\{ \ave{\nu,f} + V_T({\mathcal G}' | f_1, \ldots,f_{t-1},f) \right\}
\end{align}
Following the formulation of Rakhlin et al.~\cite{RakhlinRL}, we say that a sequence of functions $\wh{V}_T({\mathcal G}' | f_1, \ldots,f_{t})$ is an {\em admissible relaxation} if for any $f_1, \ldots, f_t \in \cF$,
\begin{subequations}\label{eq: rakhlinadmiss}
\begin{align} 
	\wh{V}_T({\mathcal G}' | f_1, \ldots,f_{t}) &\ge \inf_{\nu \in {\mathcal G}'} \sup_f \left\{ \ave{\nu,f} + \wh{V}_T({\mathcal G}' | f_1, \ldots,f_{t},f)\right\}, \qquad t = T-1,\ldots,0 \\
	\wh{V}_T({\mathcal G}' | f^T) &\ge - \inf_{\nu \in {\mathcal G}'} \sum^T_{t=1} \ave{\nu,f_t}.
\end{align}
\end{subequations}
We can associate a behavioral strategy to any admissible relaxation as follows:
\begin{align*}
	\wh{\gamma}_t(x,f^{t-1}) &= \nu_t = \argmin_{\nu \in {\mathcal G}'} \sup_{f \in \cF}\left\{ \ave{\nu,f} + \wh{V}_T({\mathcal G}' | f_1, \ldots,f_{t-1},f)\right\}.
\end{align*}
In fact, as pointed out by Rakhlin et al.~\cite{RakhlinRL}, exact minimization is unnecessary: any choice $\nu_t = \wh{\gamma}_t(x,f^{t-1})$ that satisfies
$$
\wh{V}_T({{\mathcal G}'}|f_1,\ldots,f_{t-1}) \ge \sup_{f \in \cF}\left\{ \ave{\nu_t,f} + \wh{V}_T({\mathcal G}' | f_1, \ldots,f_{t-1},f)\right\},
$$
is admissible. The above admissibility condition of \cite{RakhlinRL} is different from \eqref{eq:condtionvalueshapley} in the sense that there is no conditioning on the state variable. It is also different from \eqref {eq:generalrelaxation1} because it is not indexed by the state variable.  The following theorem provides the main regret bound for the convex-analytic approach:
\begin{theorem} \label{thm:main_2nd} Suppose that the MDP is unichain and the environment is oblivious. Then, for any family of admissible relaxations given by \eqref{eq: rakhlinadmiss} and the corresponding behavioral strategy \ $\wh{\bd{\gamma}}$, we have
	\begin{align} \label{eq:thm2eq2}
		R^{\wh{\bd{\gamma}},\bd{f}}_x({\cal G'})\deq\E^{\wh{\bd{\gamma}},\bd{f}}_x\left\{\sum^T_{t=1}f_t(X_t,U_t) - \inf_{P \in \cM({\cal G}')}\E\left[\sum^T_{t=1}f_t(X,U)\right]\right\} &\le \wh{V}_T({\mathcal G}' | {\mathsf e}) + C_\cF  \sum^T_{t=1} \| \mu^{\wh{\bd{\gamma}},\bd{f}}_t - \pi^{\wh{\bd{\gamma}},\bd{f}}_t \|_1
	\end{align}
	where $C_\cF = \sup_{f \in \cF} \| f \|_\infty$. 
\end{theorem}
\begin{proof} 
See Appendix~\ref{app:main_2nd}.
 \end{proof}

\section{Example derivations of explicit algorithms}

In the preceding section, we have described two different approaches to construct relaxations and algorithms for online MDPs. Specifically, the value-function approach make use of Poisson inequalities for MDPs \cite{MeynTweedie} to reduce the online MDP problem to a collection of standard online learning problems, each of which is associated with a particular state. We need to construct a separate relaxation for each of these problems. The convex-analytic approach reduces the online MDP problem to an online linear optimization problem, and uses a single relaxation (no longer indexed by the state) to derive algorithms. The common property of these two approaches is that we can apply any algorithm derived in the simpler static setting to the original dynamic setting and automatically bound its regret.

\subsection{The value-function approach}
\label{sec:derivealgo}

In this section, we apply the value-function approach to recover and derive a class of online MDP algorithms \cite{EvenDar, Yu}. The common thread running through this class of algorithms is that value functions have to be computed in order to get the policy for each time step. The strategies derived in this section using our general framework also belong to a class of algorithms for {\em online prediction with expert advice} \cite{PLG}. In this setting, the agent combines the recommendations of several individual ``experts'' into an overall strategy for choosing actions in real time in response to causally revealed information. Every expert is assigned a ``weight'' indicating how much the agent trusts that expert, based on the previous performance of the experts.  One of the more popular algorithms for prediction with expert advice is the Randomized Weighted Majority (RWM) algorithm, which updates the expert weights multiplicatively \cite{LittlestoneWM}. It has an alternative interpretation as a Follow the Regularized Leader (FRL) scheme \cite{ShaiSinger}: The weights chosen by an RWM algorithm minimize a combination of empirical cost and an entropic regularization term. The entropy term (equal to the divergence between the current weight distribution and the uniform distribution over the experts) penalizes ``spiky'' weight vectors, thus guaranteeing that every expert has a nonzero probability of being selected at every time step, which in turn provides the algorithm with a degree of stability. The common feature of the strategies we consider in this section is that RWM algorithms are applied in parallel for each state.

We start by recovering an expert-based algorithm for online MDPs. Similar to our set-up, Even-Dar et al.~\cite{EvenDar} consider an MDP with arbitrarily varying cost functions. The main idea of their work is to efficiently incorporate existing expert-based algorithms \cite{LittlestoneWM, PLG} into the MDP setting. For an MDP with state space $\sX$ and action space $\sU$, there are $\lvert \sU \rvert^{\lvert \sX \rvert}$ deterministic Markov policies (state feedback laws), which renders the obvious approach of associating an expert with each possible deterministic policy computationally infeasible. Instead, they propose an alternative efficient scheme that works by associating a separate expert algorithm to each {\em state}, where experts correspond to {\em actions} and the feedback to provided each expert algorithm depends on the aggregate policy determined by the action choices of all the individual algorithms. Under a unichain assumption similar to the one we have made above, they show that the expected regret of their algorithm is sublinear in $T$ and independent of the size of the state space. Their algorithm can be summarized as follows:
\begin{center}
\begin{tabular}{|l|}
\hline
Put in every state $x$ an expert algorithm ${\cal A}_x$\\
for $t = 1, 2, \ldots$ do \\
\ \ Let $P_t(\cdot|x_t)$ be the distribution over actions of ${\cal A}_{x_t}$\\
\ \ Use policy $P_t$ and obtain $f_t$ from the environment\\
\ \ For every $x \in \sX$ \\
\ \ \ \ Feed ${\cal A}_x$ with loss function $\wh{Q}_{P_t,f_t} (x, \cdot) =  \E\left[\sum^{\infty}_{i=0}\left(f_t(X_i, U_i) - \eta^{P_t,{\bd f}}_t \right)\right]$,\\
\ \ \ \ where $\E$ is taken w.r.t.\ the Markov chain induced by $P_t$ from the initial state $x$, \\
\ \ \ \ and $\eta^{P_t,\bd{f}}_t$ is the steady-state cost $\ave{\pi_{P_t} \otimes P_t,f_t}$ \\
\ \ end for\\
end for\\
\hline 
\end{tabular}
\end{center}
As we show next, the algorithm proposed by \cite{EvenDar} arises from a particular relaxation under the value-function approach.  For every possible state value $x \in \sX$, we want to construct an admissible relaxation that satisfies \eqref{eq:generalrelaxation1}. Here we show that the relaxation can be obtained as an upper bound of a quantity called {\em conditional sequential Rademacher complexity}, which is defined by Rakhlin et al.~\cite{RakhlinRL} as follows. Let $\eps$ be a vector $(\eps_1, \ldots, \eps_T)$ of i.i.d.\ Rademacher random variables, i.e., $\Pr(\eps_i = \pm 1) = 1/2$. For a given $x \in \sX$, an $\cH_x$-valued tree $\mathbf h$ of depth $d$ is defined as a sequence $(\mathbf h_1, \ldots, \mathbf h_d)$ of mappings $\mathbf h_t: \{ \pm 1\}^{t-1} \to \cH_x$, where $\cH_x$ is the function class defined in Section~\ref{ssec:value_function_approach}. Then the conditional sequential Rademacher complexity at state $x$ is defined as
\begin{align*} 
\mathcal R_{x,t}(h^t_x) = \sup_{\mathbf h} \E_{\eps_{t+1:T}} \max_{u \in \sU} \left[ 2\sum^T_{s=t+1}\eps_s \left[\mathbf h_{s-t}(\eps_{t+1:s-1})\right] (u) - \sum^t_{s=1} h_{x,s}(u) \right], \qquad \forall h^t_x \in \cH^t_x.
\end{align*}
Here the supremum is taken over all $\cH_x$-valued binary trees of depth $T-t$. The term containing the tree $\mathbf h$ can be seen as ``future", while the term being subtracted off can be seen as ``past". This quantity is conditioned on the already observed $h^t_x$, while for the future we consider the worst possible binary tree. As shown by \cite{RakhlinRL}, this Rademacher complexity is itself an admissible relaxation for standard (state-free) online optimization problems; moreover, one can obtain other relaxations by further upper-bounding the Rademacher complexity. As we will now show, because the action space $\sU$ is finite and the functions in $\cH_x$ are uniformly bounded (Assumption~\ref{as:existboundQ}), the following upper bound on $\mathcal R_{x,t}(\cdot)$ is an admissible relaxation, i.e., it satisfies condition \eqref{eq:generalrelaxation1}:
\begin{align} \label{eq:exprelax1}
\wh W_{x,t}(h^t_x) =  \rho \log \left( \sum_{u \in \sU}\exp\left(- \frac{1}{\rho} \sum^t_{s=1} h_{x,s}(u)\right)\right) + \frac{2}{\rho}(T-t) L(\sX,\sU, \cF)^2,
\end{align}
where the learning rate $\rho > 0$ can be tuned to optimize the resulting regret bound. This relaxation leads to an algorithm that turns out to be exactly the scheme proposed by \cite{EvenDar}:

\begin{proposition}\label{pps:recoverevendar} The relaxation \eqref{eq:exprelax1} is admissible and it leads to a recursive exponential weights algorithm, specified recursively as follows: for all $x \in \sX$, $u \in \sU$
\begin{align} \label{eq:recursiveevendar}
P_{t+1} (u|x) = \frac{ P_t(u|x) \exp\left(-\frac{1}{\rho}h_{x,t}(u)\right)}{\Ave{P_t(\cdot|x), \exp\left(-\frac{1}{\rho}h_{x,t}\right)}} = \frac{ \nu_1(u) \exp\left(-\frac{1}{\rho}\sum^t_{s=1}h_{x,s}(u)\right)}{\Ave{\nu_1, \exp\left(-\frac{1}{\rho}\sum^t_{s=1}h_{x,s}\right)}},\qquad t = 0,\ldots,T-1
\end{align}
where $\nu_1$ is the uniform distribution on $\sU$. 
\end{proposition}

\begin{proof} See Appendix~\ref{app:recoverevendar}.\end{proof}

The above algorithm works with any collection of $\wh{Q}$ functions satisfying the reverse Poisson inequalities determined by the $f_t$'s (recall Assumption~\ref{as:existboundQ}). Here is one particular example of such a function --- the usual $Q$-function that arises in reinforcement learning and that was used by \cite{EvenDar}. Recall our assumption that every randomized state feedback law $P \in \cM(\sU|\sX)$ has a unique stationary distribution $\pi_P$. For given choices of $P \in \cM(\sU|\sX)$ and $f \in \cF$, consider the function 
$$
\wh{Q}_{P,f} (x, u) = \lim_{T \rightarrow \infty} \E_P \left[ \sum^T_{t=1} f(X_t, U_t) - \ave{\pi_P \otimes P, f}\Bigg | X_1 = x, U_1 = u \right],
$$ 
where $X_t$ and $U_t$ are the state and action at time step $t$ after starting from the initial state $X_1 = x$, applying the immediate action $U_1 = u$, and following $P$ onwards. It is easy to check that $\wh{Q}_{P,f} (x, u)$ satisfies the reverse Poisson inequality for $P$ with forcing function $f$. In fact, it satisfies \eqref{eq:RPI} with equality. We can also derive a bound on the Q-function in terms of the mixing time $\tau$. Let us first bound $\wh{Q}_{P,f}(x, P)$ where $P$ is used on the first step instead of $u$. For all $t$, let $\mu^{P,{f}}_{x,t}$ be the state distribution at time $t$ starting from $x$ and following $P$ onwards. So we have
	\begin{align*}
	\wh{Q}_{P,f}(x, P)=\lim_{T \rightarrow \infty} \sum^T_{t=1}\left[\ave{\mu^{P,{f}}_{x,t} \otimes P,f} - \ave{\pi_P \otimes P,f}\right] 
		&\le \| f\|_\infty \sum^T_{t=1}  \| \delta_xP^t - \pi_P P^t \|_1 \\
		&\le 2 \| f\|_\infty  \sum^T_{t=1} e^{-t/\tau} \\
		&\le 2\tau \| f\|_\infty,
	\end{align*}
where $\delta_x \in \cP(\sX)$ is the Dirac distribution centered at $x$, and the first inequality results from repeated application of the uniform mixing bound \eqref{eq:mixingtimeeq}.  Due to the fact that the one-step cost is bounded by $C_\cF = \sup_{f \in \cF} \| f \|_\infty$, we have
$$
\wh{Q}_{P,f}(x,u) \le \wh{Q}_{P,f}(x,P) + f(x,u) - \ave{\mu^{P,{f}}_{x,1} \otimes P,f} \le 2 \tau C_\cF + C_\cF \le 3 \tau C_\cF.
$$
We can now establish the following regret bound for the exponential weights strategy \eqref{eq:recursiveevendar}:

\begin{theorem}\label{thm:evendarbound} Let $L \deq L(\sX,\sU, \cF)$. Assume the uniform mixing condition is satisfied. Then for the relaxation \eqref{eq:exprelax1} and the corresponding behavioral strategy $\wh{\bd{\gamma}}$ given by \eqref{eq:recursiveevendar} with $\rho = \sqrt{\frac{2TL^2}{\log|\sU|}}$, we have 
\begin{align*} 
\E^{\wh{\bd{\gamma}},\bd{f}}_x \left[\sum^T_{t=1}f_t(X_t,U_t) - \Psi(\bd{f})\right] \le 2 L\sqrt{2 T \log |\sU|} + C_{\cF}(\tau+1)^2 \sqrt{\frac{\log|\sU|T}{2}} + (2\tau+2) C_{\cF}.
\end{align*}
\end{theorem}
\begin{proof} See Appendix~\ref{app:evendarbound}. \end{proof}
As we can see, this regret bound is consistent with the bound derived in \cite{EvenDar}.
Therefore, we have shown that our framework, with a specific choice of relaxation, can recover their algorithm. The advantage of our general framework is that we can analyze the part of the corresponding regret bound simply by instantiating our analysis on specific relaxations, without the need of ad-hoc proof techniques applied in \cite{EvenDar}.

The above policy relies on exponential weight updates. We now present a ``lazy" version of that policy, wherein time is divided into phases of increasing length, and during each phase the agent applies a fixed state feedback law. The main advantage of lazy strategies is their computational efficiency, which is the result of a looser relaxation and hence suboptimal scaling of the regret with the time horizon. 

We partition the set of time indices $1,2,\ldots$ into nonoverlapping contiguous phases of (possibly) increasing duration. The phases are indexed by $m \in \Naturals$, where we denote the $m$th phase by $\cT_m$ and its duration by $\tau_m$. We also define $\cT_{1:m} \deq \cT_1 \cup \ldots \cup \cT_m$ (the union of phases $1$ through $m$) and denote its duration by $\tau_{1:m}$. Let $M \le T$ denote the number of complete phases concluded before time $T$. Here we need a describe a generic algorithm that works in phases:

\begin{center}
\begin{tabular}{|l|}
\hline
Initialize at $t = 0$ and phases $\cT_1, \ldots, \cT_M$ s.t. $\tau_{1:M} = T$ \\
For $t \in \cT_1$, choose $u_t$ uniformly at random over $\sU$ \\
for $m = 2, 3, \ldots$ \\
\ \  for $t \in \cT_m$ do \\
\ \ \ \ if the process is at state $x$, choose action  $u_t$ randomly according to $P_{m} (u|x)$ \\
\ \ \ \ where $P_{m} (u|x)$ is the state feedback law only using information from phase 1 to $m-1$\\
\ \ end for \\
end for\\
\hline 
\end{tabular}
\end{center}

Because now we work in phases instead of time steps, we need to provide an alternative definition of relaxations and admissibility condition. For every state $x \in \sX$, we denote by $h^m_x$  the $\tau_m$-tuple $(h_{x,s} : s \in \cT_m)$, and by $h_{x,1:m}$ the $\tau_{1:m}$-tuple $(h_{x,1}, h_{x,2}, \ldots, h_{x,\tau_{1:m}})$.  For each $x \in \sX$, we will say that a sequence of functions $\wh W_{x,m}: \cH_x^{\tau_{1:m}} \rightarrow \Reals, m = 1, \ldots, M,$ is an admissible relaxation if
\begin{align*}
\wh W_{x,M} (h_{x,1:M}) &\ge - \inf_{\nu \in \cP(\sU)} \E_{U \sim \nu} \left[ \sum^T_{t=1} h_{x,t}(U)\right] \\
\wh W_{x,m} (h_{x,1:m}) &\ge \inf_{\nu \in \cP(\sU)} \sup_{h^m_x \in \cH^{\tau_m}_x} \left\{\E_{U \sim \nu} \left[ \sum_{s \in \cT_m}h_{x,s}(U)\right] + \wh W_{x,m+1} (h_{x,1:m},h^{m+1}_x)\right\}, \quad m = M-1, \ldots, 1
\end{align*}
For a given state $x$, we also define the conditional sequential Rademacher complexity in terms of phases:
\begin{align*} 
\mathcal R_{x,m}(h_{x,1:m}) = \sup_{\mathbf h} \E_{\eps_{m+1:M}} \max_{u \in \sU} \left[ 2\sum^M_{j=m+1}\eps_j \sum_{t \in \cT_j}\left[\mathbf h_{x,t}(\eps)\right] (u) - \sum^m_{i=1} \sum_{s \in \cT_i}h_{x,s}(u) \right].
\end{align*}
Here the supremum is taken over all $\cH_x$-valued binary trees of depth $M-m$. When recovering the method in \cite{EvenDar}, we replaced the actual future induced by the infimum and supremum pairs in the conditional value by the ``worst future" binary tree, which involves expectation over a sequence of coin flips in every time step. By contrast, in the above quantity we replace the real future by the ``worst future" binary tree that branches only once per phase. Now we can construct the following relaxation:
\begin{align}\label{eq:modrelax}
\wh W_{x,m}(h_{x,1:m}) =  \rho \log \left( \sum_{u \in \sU}\exp\left(- \frac{1}{\rho} \sum^m_{i=1} \sum_{s \in \cT_i}h_{x,s}(u)\right)\right) + \frac{2L(\sX,\sU, \cF)^2}{\rho} \sum^{M}_{j=m+1} \tau_j^2 .
\end{align}
The corresponding algorithm, specified in \eqref{eq:recursiveyu} below, uses a fixed state feedback law throughout each phase:

\begin{proposition}\label{pps:recoveryu} The relaxation \eqref{eq:modrelax} is admissible and it leads to the following Markov policy for phase $m$:
\begin{align} \label{eq:recursiveyu}
P_{m} (u|x) = \frac{ \nu_1(u) \exp\left(-\frac{1}{\rho}\sum^{m-1}_{i=1} \sum_{s \in \cT_i} h_{x,s}(u)\right)}{\Ave{\nu_1, \exp\left(-\frac{1}{\rho}\sum^{m-1}_{i=1} \sum_{s \in \cT_i} h_{x,s}\right)}},
\end{align}
where $\nu_1$ is the uniform distribution on $\sU$. 
\end{proposition}
\begin{proof} See Appendix~\ref{app:recoveryu}. \end{proof}
 Now we derive the regret bound for \eqref{eq:recursiveyu}:

\begin{theorem}\label{thm:yubound} Let $L \deq L(\sX,\sU, \cF)$. Under the same assumptions as before, the behavioral strategy $\wh{\bd{\gamma}}$ corresponding to  \eqref{eq:recursiveyu} enjoys the following regret bound when $\rho = \sqrt{\frac{2\sum^M_{i=1}\tau_i^2L^2}{\log|\sU|}}$:
\begin{align}\label{eq:yuregret}
\E^{\wh{\bd{\gamma}},\bd{f}}_x \left[\sum^T_{t=1}f_t(X_t,U_t) - \Psi(\bd{f})\right] \le 2 L\sqrt{2 \log |\sU| \sum^M_{i=1}\tau_i^2} +  \frac {2C_{\cF}M}{1-e^{-1/\tau}}.
\end{align}
\end{theorem}
\begin{proof} See Appendix~\ref{app:yubound}. \end{proof}
Our behavioral strategy \eqref{eq:recursiveevendar} is a novel randomized weighted majority (RWM) algorithm for online MDPs. Yu et al. \cite{Yu} also consider a similar model, where the decision-maker has full knowledge of the transition kernel, and the costs are chosen by an oblivious (open-loop) adversary. They propose an algorithm that computes and changes the policy periodically according to a perturbed version of the empirically observed cost functions, and then follows the computed stationary policy for increasingly long time intervals. As a result, their algorithm achieves sublinear regret and has diminishing computational effort per time step; in particular, it is computationally more efficient than that of \cite{EvenDar}.

Although our new algorithm is similar in nature to the algorithm of \cite{Yu}, it has several advantages.  First, in the algorithm of \cite{Yu}, the policy computation at the beginning of each phase requires solving a linear program and then adding a carefully tuned random perturbation to the solution. As a result, the performance analysis in \cite{Yu} is rather lengthy and technical (in particular, it invokes several advanced results from perturbation theory for linear programs). By contrast, our strategy is automatically randomized, and the performance analysis is a lot simpler. Second, the regret bound of Theorem~\ref{thm:yubound} shows that we can control the scaling of the regret with $T$ by choosing the duration of each phase, whereas the algorithm of   \cite{Yu} relies on a specific choice of phase durations in order to guarantee that the regret is sublinear in $T$ and scales as $O(T^{3/4})$. We show that if the horizon $T$ is known in advance, then it is possible to choose the phase durations to secure $O(T^{2/3})$ regret, which is better than the $O(T^{3/4})$ bound derived by \cite{Yu}.

\begin{corollary}\label{cor:optbound} Consider the setting of Theorem~\ref{thm:yubound}. For a given horizon $T$, the optimal choice of phase lengths is $T^{1/3}$, which gives the regret of \ $O(T^{2/3})$.
\end{corollary}
\begin{proof} See Appendix~\ref{app:optbound}. \end{proof}

\subsection{The convex-analytic approach}
\label{sec:derivedick}

In this section, we use the convex-analytic approach to derive an algorithm that relies on the reduction of the online MDP problem to an online linear optimization problem over the state-action polytope [recall the definition in Eq.~\eqref{eq:SAP}]. Structurally, this algorithm is similar to the Online Mirror Descent scheme proposed and analyzed recently by Dick et al.~\cite{DickCsaba}; however, its key ingredients and the resulting performance guarantee on the regret are more closely related to interior-point methods of Abernethy et al.~\cite{AbernethyIPM}. Moreover, we will show that this algorithm arises from an admissible relaxation with respect to \eqref{eq: rakhlinadmiss}.

We start by introducing the definition of a \textit{self-concordant barrier}, which is basic to the theory of interior point methods \cite{Nesterov,NemTodd}: Let $\cK \subseteq \Reals^n$ be a closed convex set with nonempty interior. A function $F : {\rm int}(\cK) \to \Reals$ is a \textit{barrier} on $\cK$ if $F(x_i) \to +\infty$ along any sequence $\{x_i\}^\infty_{i=1} \subset \cK$ that converges to a boundary point of $\cK$. Moreover, $F$ is \textit{self-concordant} if it is a convex $C^3$ function, such that the inequality
	$$
	\nabla^3F(v)[h,h,h] \le 2\left(\nabla^2 F(v)[h,h]\right)^{3/2}
	$$
	holds for all $v \in {\rm int}(\cK)$ and $h \in \Reals^n$. Here, $\nabla^2 F(v)$ and $\nabla^3 F(v)$ are the Hessian and the third-derivative tensor of $F$ at $v$, respectively. We also need some geometric quantities induced by $F$. The first is the \textit{Bregman divergence} $D_F : {\rm int}(\cK) \times {\rm int}(\cK) \to \Reals^+$, defined by
\begin{align}\label{eq:Bregman}
	D_F(v,w) \deq F(v) - F(w) - \langle \nabla F(w),v-w\rangle, \qquad v,w \in {\rm int}(\cK).
\end{align}
The second is the \textit{local norm} of $h \in \Reals^n$ around a point $v \in {\rm int}(\cK)$ (assuming $\nabla^2 F(v)$ is nondegenerate):
$$
\| h \|_v \deq \sqrt{\nabla^2 F(v)[h,h]}.
$$
Finally, if $F$ is self-concodrant, then so is its \textit{Legendre--Fenchel dual} $F^*(h) \deq \sup_{v \in {\rm int}(\cK)} \left\{ \langle h,v\rangle - F(v)\right\}$. Thus, the definitions of the Bregman divergence and the local norm carry over to $F^*$. Specifically,
$$
\| f \|^*_h \deq \sqrt{\ave{f, \nabla^2 F^*(h)f}} \equiv \sqrt{\nabla^2 F^*(h)[f,f]}
$$
is the local norm of $f$ at $h$ induced by $F^*$ (by the following assumption that $F^*$ is strictly convex, this local norm is well-defined everywhere). 
 
Both our algorithm and the relaxation that induces it revolve around a self-concordant barrier for the set $\cK = {\cal G}$, the state-action polytope of our MDP. This set is a compact convex subset of $\Reals^{|\sX| \times |\sU|}$ with nonempty interior. We make the following assumption:
\begin{assumption}\label{as:barrier} The state-action polytope ${\mathcal G}$ associated to the MDP with controlled transition law $K$ admits a self-concordant barrier $F : {\rm int}({\cal G}) \to \Reals$ with the following properties:
	\begin{enumerate}
		\item $F$ is strictly convex on ${\rm int}({\cal G})$, and its dual $F^*$ is strictly convex on $\Reals^{|\sX| \times |\sU|}$.
		\item The gradient map $\nu \mapsto \nabla F(\nu)$ is a bijection between ${\rm int}({\cal G})$ and $\Reals^{|\sX| \times |\sU|}$, and admits the map $h \mapsto \nabla F^*(h)$ as inverse.
		\item The minimum value of $F$ on ${\rm int}({\cal G})$ is equal to $0$.
	\end{enumerate}
\end{assumption}
\noindent This assumption is not difficult to meet in practice. For example, the \textit{universal entropic barrier} of Bubeck and Eldan \cite{BubeckEldan} (which can be constructed for any compact convex polytope) satisfies these requirements.

We are now ready to present our algorithm and the associated relaxation. We start by describing the former:
\begin{center}
\begin{tabular}{|l|}
\hline
For $t = 1, 2, \ldots$ do \\
\ \ If $t=1$, choose $\nu_t = \nu^* \equiv \argmin_{v \in {\rm int}({\cal G})} F(\nu)$; else choose
$\nu_t = \nabla F^*\left(\nabla F(\nu_{t-1}) - \rho f_{t-1}\right)$ \\
\ \ Construct the policy $P_t = P_{\nu_t}$ according to Eq.~\eqref{eq:from_ergodic_to_policy}\\
\ \ Observe the state $X_t$\\
\ \ Draw the action $U_t \sim P_t(\cdot|X_t)$ and obtain $f_t$ from the environment\\
end for\\
\hline 
\end{tabular}
\end{center}
Here, $\rho > 0$ is the tunable learning rate. Note also that, by virtue of Assumption~\ref{as:barrier}, the sequence of measures $\{\nu_t\}$ lies in ${\rm int}({\cal G})$. Next, we describe the relaxation. For reasons that will be spelled out shortly, we focus on the regret with respect to policies induced by elements of a given subset ${\cal G}'$ of ${\rm int}({\cal G})$. For $t=0,\ldots,T$, we let
\begin{align} \label{eq:OMDrelax}
\wh{V}_T({\cal G}'|f_1,\ldots,f_t) = \sup_{\mu \in {\cal G}'} \left\{\sum^t_{s=1} \ave{\mu,-f_s} + \frac{1}{\rho} D_{F} (\mu, \nu_{t+1})\right\} + 2\rho (T-t),
\end{align}
Note that $\nu_{t+1}$ is a deterministic function of $f_1,\ldots,f_t$, and therefore the relaxation is well-defined.

\begin{proposition}\label{pps:OMDrelaxadmin} Suppose that the learning rate $\rho$ is such that $\rho \| f_t \|^*_{\nabla F(\nu_t)} \le 1/2$ for all $t = 1,\ldots,T$. Assume that $\| f_t \|^*_{\nabla F(\nu_t)} \le 1$, for all $t = 1,\ldots,T$. The relaxation \eqref{eq:OMDrelax} is admissible, and the algorithm that generates the sequence $\{\nu_t\}$ is also admissible:
$$
\wh{V}_T({\cal G}'|f_1,\ldots,f_{t-1}) \ge \sup_{f \in \cF}\left\{ \ave{\nu_t,f} + \wh{V}_T({\mathcal G}' | f_1, \ldots,f_{t-1},f)\right\}
$$ \end{proposition}
\begin{proof}
See Appendix~\ref{app:OMDrelaxadmin}.
\end{proof}

Here we impose the assumption that $\rho \| f_t \|^*_{\nabla F(\nu_t)} \le 1/2$ for all $t = 1,\ldots,T$. A restriction of this kind is necessary when using interior-point methods to construct online optimization schemes --- see, for example, the condition of Theorem~4.1 and 4.2 in \cite{AbernethyIPM}. The boundedness of the dual local norm $\| f_t \|^*_{\nabla F(\nu_t)}$ is a reasonable assumption as well. In particular, Abernethy et al. \cite{AbernethyIPM} points out that if a large number of the points $\nu_t$ are close to the boundary of ${\cal G}'$, then the eigenvalues of the Hessian of $F$ at those points will be large due to the large curvature of the barrier near the boundary of ${\cal G}'$. This will imply, in turn, that the dual local norm $\| f_t \|^*_{\nabla F(\nu_t)}$ is expected to be small. 

Now we are ready to present the online-learning (i.e., steady-state) part of the regret bound for the above algorithm:
\begin{theorem}\label{thm:OMDbound1} Let $D_F({\cal G'}) \deq\sup_{\nu \in {\cal G}'} D_{F}(\nu,\nu_1)$. Suppose that the learning rate $\rho$ is such that $\rho \| f_t \|^*_{\nabla F(\nu_t)} \le 1/2$ for all $t = 1,\ldots,T$. Assume that $\| f_t \|^*_{\nabla F(\nu_t)} \le 1$, for all $t = 1,\ldots,T$. Then for the relaxation \eqref{eq:OMDrelax} and the corresponding algorithm, we can bound the online learning part of the regret as
\begin{align}\label{eq:OMD_regret}
\sum^T_{t=1} \ave{\nu_t,f_t} - \inf_{\nu \in {\cal G}'} \sum^T_{t=1} \ave{\nu,f_t} \le   \frac{D_F({\cal G}')}{\rho}  + 2\rho T.
\end{align} 
\end{theorem}

\begin{proof}
See Appendix~\ref{app:OMDbound1}.
\end{proof}
\begin{remark} {\em Since $F$ is a barrier, $D_F({\cal G'})$ will be finite only if all the elements of ${\cal G}'$ are not too close to the boundary of ${\cal G}$. This motivates our restriction of the comparator term to a proper subset ${\cal G}' \subset {\rm int}({\cal G})$.}
\end{remark}

Finally, we present the total regret bound for the above algorithm, including the stationarization error:
\begin{theorem}\label{thm:OMDbound2} Suppose that all of our earlier assumptions are in place, and also that the uniform mixing condition is satisfied. Then for the relaxation \eqref{eq:OMDrelax} and the corresponding  algorithm, we have 
\begin{align} 
&\E^{\wh{\bd{\gamma}},\bd{f}}_x \left[\sum^T_{t=1}f_t(X_t,U_t) - \inf_{P \in \cM({\cal G'})}\right] \nonumber\\
& \qquad \qquad \le   \frac{D_F({\cal G}')}{\rho}  + 2\rho T + C_{\cF}(\tau+1)^2 T\Delta_T  + (2\tau +2) C_{\cF},\label{eq:MD_total_regret}
\end{align}
where $\Delta_T \deq \max_{1 \le t \le T}\max_{x \in \sX} \| P_{t-1}(\cdot|x) - P_t(\cdot|x)\|_1$.
\end{theorem}
\begin{proof}
See Appendix~\ref{app:OMDbound2}.
\end{proof}
The third term on the right-hand side of \eqref{eq:MD_total_regret} quantifies the drift of the policies generated by the algorithm. A similar term appears in all of the regret bounds of Dick et al.\ (see, e.g., the bound of Lemma~1 in \cite{DickCsaba}). Moreover, just like the Mirror Descent scheme of \cite{DickCsaba}, our  algorithm may run into implementation issues, since in general it may be difficult to compute the gradient mappings $\nabla F$ and $\nabla F^*$ associated to the self-concordant barrier $F$. We refer the reader to the discussion in the paper by Bubeck and Eldan \cite{BubeckEldan} pertaining to computational feasibility of their universal entropic barrier.

\section{Conclusions}
\label{sec:clc}
We have provided a unified viewpoint on the design and the analysis of online MDPs algorithms, which is an extension of a general relaxation-based approach of \cite{RakhlinOR} to a certain class of stochastic game models. We have unified two distinct categories of existing methods (those based on the relative-function approach and those based on the convex-analytic approach) under a general framework. We have shown that an algorithm previously proposed by \cite{EvenDar} naturally arises from our framework via a specific relaxation. Moreover, we have shown that one can obtain lazy strategies (where time is split into phases, and a different stationary policy is followed in each phase) by means of relaxations as well. In particular, we have obtained a new strategy, which is similar in spirit to the one previously proposed by \cite{Yu}, but with several advantages, including better scaling of the regret. The above two algorithms are based on the relative-function approach via reverse Poisson inequalities. Finally, using a different type of a relaxation, we have derived another algorithm for online MDPs, which relies on interior-point methods and belongs to the class of algorithms derived using the convex-analytic approach. The takeaway point is that our general technique of constructing relaxations after a stationarization step brings all of the existing methods under the same umbrella and paves the way toward constructing new algorithms for online MDPs.

\section*{Acknowledgement}
This work was supported by NSF grant CCF-1017564 and by AFOSR grant FA9550-10-1-0390. The authors are grateful to Profs.~Alexander Rakhlin and Karthik Sridharan for helping us construct the relaxation presented in Section~\ref{sec:derivedick}.

\appendix

\section{Proof of Proposition~\ref{pps:strategyextensive}}
\label{app:strategyextensive}

The agent's closed-loop behavioral strategy $\bd{\gamma}$ is a tuple of mappings $\gamma_t : \cF^{t-1} \to \cP(\sU), 1 \le t \le T$; the environment's open-loop behavior strategy $\bd{f}$ is a tuple of functions $(f_1,\ldots,f_T)$ in $\cF$. Thus,
\begin{align*}
	V(x) &= \inf_{\bd{\gamma}} \sup_{\bd{f}} \E^{\bd{\gamma},\bd{f}}_x \left[ \sum^T_{t=1}f_t(X_t,U_t) - \Psi(\bd{f})\right] \\
	&= \inf_{\gamma_1} \ldots \inf_{\gamma_T} \sup_{f_1} \ldots \sup_{f_T} \E^{\gamma_1,\ldots,\gamma_T,f_1,\ldots,f_T}_x \left[ \sum^T_{t=1}f_t(X_t,U_t) - \Psi(\bd{f})\right].
\end{align*}
We start from the final step $T$ and proceed by backward induction. Assuming $\gamma_1, \ldots, \gamma_{T-1}$ were already chosen, we have
\begin{align*}
&	\inf_{\gamma_T}\sup_{f_1, \ldots, f_T} \E^{\gamma^{T-1}, \gamma_T, f^{T-1}, f_T}_x \left\{ \sum^{T-1}_{t=1}\left[f_t(X_t,U_t)\right] + f_T(X_T, U_T)- \Psi(f^T)\right\}\\ 
&= \inf_{\gamma_T}\sup_{f_1, \ldots, f_{T-1}} \sup_{f_T} \left\{ \E^{\gamma^{T-1}, f^{T-1}}_x  \left(\sum^{T-1}_{t=1}\left[f_t(X_t,U_t)\right] \right)+ \E^{\gamma^{T-1}, \gamma_T, f^{T-1}, f_T}_x \left[f_T(X_T, U_T)- \Psi(f^T) \right] \right\}\\
&= \inf_{\gamma_T}\sup_{f_1, \ldots, f_{T-1}} \left\{ \E^{\gamma^{T-1}, f^{T-1}}_x  \left(\sum^{T-1}_{t=1}\left[f_t(X_t,U_t)\right] \right)+ \sup_{f_T}\, \E^{\gamma^{T-1}, \gamma_T, f^{T-1}, f_T}_x \left[f_T(X_T, U_T)- \Psi(f^T) \right] \right\} \\
&= \sup_{f_1, \ldots, f_{T-1}} \left\{ \E^{\gamma^{T-1}, f^{T-1}}_x  \left(\sum^{T-1}_{t=1}\left[f_t(X_t,U_t)\right] \right)+ \inf_{P_T(U_T | X_T)}\sup_{f_T}\, \E^{\gamma^{T-1}, \gamma_T, f^{T-1}, f_T}_x \left[f_T(X_T, U_T)- \Psi(f^T) \right] \right\}.
\end{align*}
The last step is due to the easily proved fact that, for any two sets $A,B$ and bounded functions $g_1 : A \to \Reals$, $g_2 : A \times B \to \Reals$,
$$
\inf_{\gamma: A \to B} \sup_a \left\{ g_1(a) + g_2(a, \gamma(a))\right\} = \sup_a \left[ g_1(a) + \inf_{b \in B} g_2(a,b)\right]
$$
(see, e.g., Lemma~1.6.1 in \cite{Bertsekas}). Proceeding inductively in this way, we get \eqref{eq:extensive2}.

\section{Proof of Proposition~\ref{pps:admissbound1}}
\label{app:admissbound1}

The proof is by backward induction. Starting at time $T$ and using the admissibility condition \eqref{eq:condtionvalueshapley}, we write
\begin{align*}
	&	\E^{\wh{\bd{\gamma}},\bd{f}}_x\left[\sum^T_{t=1}f_t(X_t,U_t)-\Psi(\bd{f})\right] \nonumber\\
	& \le 	\E^{\wh{\bd{\gamma}},\bd{f}}_x\left[\sum^T_{t=1}f_t(X_t,U_t) + \wh{V}_T(X_{T+1},f^T)\right] \\
		&= \E^{\wh{\bd{\gamma}},\bd{f}}_x\left[\sum^{T-1}_{t=1} f_t(X_t,U_t)\right] + \E^{\wh{\bd{\gamma}},\bd{f}}_x\left[f_T(X_T,U_T) + \wh{V}_T(X_{T+1},f^T)\right] \\
		&=\E^{\wh{\bd{\gamma}},\bd{f}}_x\left[\sum^{T-1}_{t=1} f_t(X_t,U_t)\right] \nonumber\\
		&\qquad + \sum_{x_T}\mu_T(x_T)\left\{ \sum_{u \in \sU}f_T(x_T,u)\left[\wh{\gamma}_T\big(x_T,f^{T-1}\big)\right](u) + \E\Big[\wh{V}_T(X_{T+1},f^T)\Big|x_T,\wh{\gamma}_T\big(x_T,f^{T-1}\big)\Big]\right\} \\
		&\le \E^{\wh{\bd{\gamma}},\bd{f}}_x\left[\sum^{T-1}_{t=1} f_t(X_t,U_t) + \wh{V}_{T-1}(X_T,f^{T-1})\right],
	\end{align*}
	where $\mu_T \in \cP(\sX)$ denotes the probability distribution of $X_T$. The last inequality is due to the fact that $\wh{\bd{\gamma}}$ is the behavioral strategy associated to the admissible relaxation $\{\wh{V}_t\}^T_{t=0}$. Continuing in this manner, we complete the proof.
	
\section{Proof of Lemma~\ref{lem:comparisonch3}}
\label{app:comparisonch3}
Let us take expectations of both sides of \eqref{eq:RPI} w.r.t.\ $\pi_{P'} \otimes P'$:
\begin{align*}
	&\ave{\pi_P \otimes P, g} - \ave{\pi_{P'} \otimes P', g} \le 	\E_{\pi_{P'} \otimes P'}\Big\{ \E[\wh{Q}(Y,P)|X,U] - \wh{Q}(X,U) \Big\}  \\
	&\qquad = \sum_{x,u} \pi_{P'}(x) P'(u|x) \Big\{ \E[\wh{Q}(Y,P)|x,u] - \wh{Q}(x,u) \Big\} \\
	&\qquad =\sum_{x,u} \pi_{P'}(x) P'(u|x) \E[\wh{Q}(Y,P)|x,u] - \sum_{x,u} \left(\sum_y \pi_{P'}(y)K(x|y,P')\right) P'(u|x)\wh{Q}(x,u) 
\end{align*}	
where in the third step we have used the fact that $\pi_{P'}$ is invariant w.r.t.\ $K(\cdot|\cdot,P')$. Then we have 
\begin{align*}	
&\sum_{x,u} \pi_{P'}(x) P'(u|x) \E[\wh{Q}(Y,P)|x,u] - \sum_{x,u} \left(\sum_y \pi_{P'}(y)K(x|y,P')\right) P'(u|x)\wh{Q}(x,u)  \\
		&\qquad = \sum_{x}\pi_{P'}(x) \left\{ \sum_u P'(u|x)\E[\wh{Q}(Y,P)|x,u] - \sum_{u,y} K(y|x,P') P'(u|y)\wh{Q}(y,u) \right\} \\
		&\qquad = \sum_x \pi_{P'}(x) \left\{ \sum_u P'(u|x)\E[\wh{Q}(Y,P)|x,u] - \sum_y K(y|x,P') \wh{Q}(y,P') \right\}, 
\end{align*}
where the second step is by definition of $\wh{Q}(y,P')$. Then we can write
\begin{align*}	
&\sum_x \pi_{P'}(x) \left\{ \sum_u P'(u|x)\E[\wh{Q}(Y,P)|x,u] - \sum_y K(y|x,P') \wh{Q}(y,P') \right\} \\
		&\qquad \stackrel{{\rm (a)}}{=} \sum_x \pi_{P'}(x) \left\{ \sum_u P'(u|x)\left(\E[\wh{Q}(Y,P)|x,u] - \sum_y K(y|x,u)\wh{Q}(y,P')\right) \right\} \\
		&\qquad = \sum_{x,u} \pi_{P'}(x)P'(u|x) \Big\{\E[\wh{Q}(Y,P)|x,u] - \E[\wh{Q}(Y,P')|x,u]\Big\} \\
		&\qquad \stackrel{{\rm (b)}}{=} \sum_{x,u,y} \pi_{P'}(x)P'(u|x)K(y|x,u) \left\{ \sum_{u'} P(u'|y)\wh{Q}(y,u') - \sum_{u'} P'(u'|y)\wh{Q}(y,u')\right\} \\
		&\qquad \stackrel{{\rm (c)}}{=} \sum_{x,y} \pi_{P'}(x) K(y|x,P') \left\{ \sum_{u'} P(u'|y)\wh{Q}(y,u') - \sum_{u'}P'(u'|y)\wh{Q}(y,u')\right\} \\
		&\qquad \stackrel{{\rm (d)}}{=} \sum_x \pi_{P'}(x) \sum_u \left[P(u|x)\wh{Q}(x,u) - P'(u|x)\wh{Q}(x,u) \right],
\end{align*}
where (a) and (c) are by definition of $K(\cdot|\cdot,P')$; (b) is by definition of $\wh{Q}(y,P')$; and in (d) we use the fact that $\pi_{P'}$ is invariant w.r.t.\ $K(\cdot|\cdot,P')$.

\section{Proof of Theorem~\ref{thm:main_ch3}}
\label{app:main_ch3}
We have
\begin{align*}
&	\E^{\wh {\bd \gamma}, \bd f}_x \left[ \sum^T_{t=1}f_t(X_t,U_t) - \Psi(\bd f)\right]\\ 
&\le \sup_{P \in \cM(\sU | \sX)} \sum^T_{t=1} \left[ \ave{\pi^{\wh {\bd \gamma}, \bd f}_t \otimes P^{\wh {\bd \gamma}, \bd f}_t,f_t} - \ave{\pi_P \otimes P, f_t}\right] + \sum^T_{t=1} \| f_t \|_{\infty} \| \mu^{\wh {\bd \gamma}, \bd f}_t - \pi^{\wh {\bd \gamma}, \bd f}_t \|_1 \\
&\le \sup_{P \in \cM(\sU | \sX)} \sum_x \pi_P(x) \sum^T_{t=1} \left( \sum_u P^{\wh {\bd \gamma}, \bd f}_t(u|x) \wh Q^{\wh {\bd \gamma}, \bd f}_t(x,u) - P(u|x) \wh Q^{\wh {\bd \gamma}, \bd f}_t(x,u)\right) + \sum^T_{t=1} \| f_t \|_{\infty} \| \mu^{\wh {\bd \gamma}, \bd f}_t - \pi^{\wh {\bd \gamma}, \bd f}_t \|_1,
\end{align*} 
where in the first equality we have used \eqref{eq:regret_two_terms}, while the second inequality is by Lemma~\ref{lem:comparisonch3}. Then we write the last term out and get 
\begin{align*}
& \sup_{P \in \cM(\sU | \sX)} \sum_x \pi_P(x) \left[\sum^{T-1}_{t=1} \left( \sum_u P^{\wh {\bd \gamma}, \bd f}_t(u|x) \wh Q^{\wh {\bd \gamma}, \bd f}_t(x,u) \right) + \sum_u P^{\wh {\bd \gamma}, \bd f}_T(u|x) \wh Q^{\wh {\bd \gamma}, \bd f}_T(x,u)-  \sum^T_{t=1}P(u|x) \wh Q^{\wh {\bd \gamma}, \bd f}_t(x,u) \right]\\
&\quad+ \sum^T_{t=1} \| f_t \|_{\infty} \| \mu^{\wh {\bd \gamma}, \bd f}_t - \pi^{\wh {\bd \gamma}, \bd f}_t \|_1\\
&\le \sup_{P \in \cM(\sU | \sX)} \sum_x \pi_P(x) \left[\sum^{T-1}_{t=1} \left( \sum_u P^{\wh {\bd \gamma}, \bd f}_t(u|x) \wh Q^{\wh {\bd \gamma}, \bd f}_t(x,u) \right) + \sum_u P^{\wh {\bd \gamma}, \bd f}_T(u|x) \wh Q^{\wh {\bd \gamma}, \bd f}_T(x,u)+\wh W_{x,T} (h^T_x) \right] \\
&\quad+ \sum^T_{t=1} \| f_t \|_{\infty} \| \mu^{\wh {\bd \gamma}, \bd f}_t - \pi^{\wh {\bd \gamma}, \bd f}_t \|_1\\
&\le \sup_{P \in \cM(\sU | \sX)} \sum_x \pi_P(x) \left[\sum^{T-1}_{t=1} \left( \sum_u P^{\wh {\bd \gamma}, \bd f}_t(u|x) \wh Q^{\wh {\bd \gamma}, \bd f}_t(x,u) \right) + \wh W_{x,T-1} (h^{T-1}_x) \right] 
+\sum^T_{t=1} \| f_t \|_{\infty} \| \mu^{\wh {\bd \gamma}, \bd f}_t - \pi^{\wh {\bd \gamma}, \bd f}_t \|_1,
\end{align*} 
where the two inequalities are by the admissibility condition \eqref{eq:generalrelaxation1}. Continuing this induction backward, and noting that $\sum^T_{t=1} \| f_t \|_{\infty} \| \mu^{\wh {\bd \gamma}, \bd f}_t - \pi^{\wh {\bd \gamma}, \bd f}_t \|_1 \le C_{\cF} \sum^T_{t=1} \| \mu^{\wh {\bd \gamma}, \bd f}_t - \pi^{\wh {\bd \gamma}, \bd f}_t \|_1$, we arrive at \eqref{eq:thm2eq1}. 

\section{Proof of Theorem~\ref{thm:main_2nd}}
\label{app:main_2nd}
Applying the same backward induction used in the proof of Proposition~\ref{pps:admissbound1} (also see \cite[Prop.~1]{RakhlinRL}), it is easy to show that
$$
\sum^T_{t=1}\ave{\nu_t,f_t} - \inf_{\nu \in {\mathcal G}'} \sum^T_{t=1} \ave{\nu,f_t} \le \wh{V}_T({\cal G}' | {\mathsf e}).
$$
Then it is straightforward to see that
\begin{align*}
&	\E^{\wh {\bd \gamma}, \bd f}_x \left\{ \sum^T_{t=1}f_t(X_t,U_t) - \inf_{P \in \cM({\cal G}')}\E\left[\sum^T_{t=1}f_t(X,U)\right]\right\}\\ 
&\le \sum^T_{t=1} \left[ \ave{\pi^{\wh {\bd \gamma}, \bd f}_t \otimes P^{\wh {\bd \gamma}, \bd f}_t,f_t} - \inf_{P \in \cM({\cal G}')}  \ave{\pi_P \otimes P, f_t}\right] + \sum^T_{t=1} \| f_t \|_{\infty} \| \mu^{\wh {\bd \gamma}, \bd f}_t - \pi^{\wh {\bd \gamma}, \bd f}_t \|_1 \\
&\le \sum^T_{t=1}\ave{\nu_t,f_t} - \inf_{\nu \in {\cal G}'} \sum^T_{t=1} \ave{\nu,f_t} + \sum^T_{t=1} \| f_t \|_{\infty} \| \mu^{\wh {\bd \gamma}, \bd f}_t - \pi^{\wh {\bd \gamma}, \bd f}_t \|_1,
\end{align*} 
where in the first equality we have used \eqref{eq:regret_two_terms}.

\section{Proof of Proposition~\ref{pps:recoverevendar}}
\label{app:recoverevendar}

First we show that the relaxation \eqref{eq:exprelax1} arises as an upper bound on the conditional sequential Rademacher complexity. The proof of this is similar to the one given by \cite{RakhlinOR}, except that they also optimize over the choice of the learning rate $\rho$. For any $\rho > 0$,
\begin{align*}
&  \E_{\eps} \left[ \max_{u \in \sU} \left\{ 2\sum^{T-t}_{i=1}\eps_i \left[\mathbf h_{i}(\eps)\right](u) - \sum^t_{s=1} h_{x,s}(u) \right\} \right]\\ 
&\le \rho \log \left( \E_{\eps} \left[ \max_{u \in \sU} \exp\left( \frac{2}{\rho} \sum^{T-t}_{i=1}\eps_i \left[\mathbf h_{i}(\eps)\right](u) - \frac{1}{\rho} \sum^t_{s=1} h_{x,s}(u) \right) \right]\right)\\
&\le \rho \log \left( \E_{\eps} \left[ \sum_{u \in \sU} \exp\left( \frac{2}{\rho} \sum^{T-t}_{i=1}\eps_i \left[\mathbf h_{i}(\eps)\right](u) - \frac{1}{\rho} \sum^t_{s=1} h_{x,s}(u) \right) \right]\right),
\end{align*} 
where the first inequality is by Jensen's inequality, while the second inequality is due to the non-negativity of exponential function. Then we pull out the second term inside the expectation $ \E_{\eps}$ and get
\begin{align*}
& \rho \log \left( \sum_{u \in \sU}\exp\left(- \frac{1}{\rho} \sum^t_{s=1} h_{x,s}(u)\right)\E_{\eps} \left[  \prod^{T-t}_{i=1} \exp\left( \frac{2}{\rho} \eps_i \left[\mathbf h_{i}(\eps)\right](u)  \right) \right]\right)\\
&\le \rho \log \left( \sum_{u \in \sU}\exp\left(- \frac{1}{\rho} \sum^t_{s=1} h_{x,s}(u)\right)\times  \exp\left( \frac{2}{\rho^2} \max_{\eps_1, \ldots, \eps_{T-t} \in \{\pm1\}}\sum^{T-t}_{i=1} \big(\left[\mathbf h_{i}(\eps)\right](u) \big)^2 \right) \right)\\
&\le \rho \log \left( \sum_{u \in \sU}\exp\left(- \frac{1}{\rho} \sum^t_{s=1} h_{x,s}(u)\right)\max_{u}  \exp\left( \frac{2}{\rho^2} \max_{\eps_1, \ldots, \eps_{T-t} \in \{\pm1\}}\sum^{T-t}_{i=1} \big(\left[\mathbf h_{i}(\eps)\right](u)\big)^2 \right) \right) \\
&\le \rho \log \left( \sum_{u \in \sU}\exp\left(- \frac{1}{\rho} \sum^t_{s=1} h_{x,s}(u)\right)\right) +\frac{2}{\rho} \sup_{\mathbf h}\max_{u \in \sU} \max_{\eps_1, \ldots, \eps_{T-t} \in \{\pm1\}}\sum^{T-t}_{i=1} \big([\mathbf h_{i}(\eps)](u) \big)^2,
\end{align*} 
where the first inequality is due to Hoeffding's lemma (see, e.g., Lemma~A.1 in \cite{PLG}) applied to the expectation w.r.t.\ $\eps$. The last term, representing the worst-case future, is upper bounded by $\frac{2}{\rho}(T-t) L(\sX,\sU, \cF)^2$. We thus obtain our exponential weight relaxation from \eqref{eq:exprelax1}.

Next we prove that the relaxation \eqref{eq:exprelax1} is admissible and leads to the recursive algorithm \eqref{eq:recursiveevendar}. To keep the notation simple, we drop the subscript $x$ in the following. In particular, we use $h_t$ for $h_{x,t}$, $\wh{W}_t$ for $\wh{W}_{x,t}$, $\nu_t$ for $P_t(\cdot|x)$, etc. The admissibility condition to be proved is
\begin{align*} 
\sup_{h_{t} \in \cH_x} \left\{\E_{U \sim \nu_t} \left[ h_{t}(U) \right] + \wh W_{t} (h^t)\right\}  \le \wh W_{t-1}(h^{t-1}).
\end{align*}
Note that
$$
\Ave{\nu_t,\exp\left(- \frac{1}{\rho} h_t\right)} = \sum_{u \in \sU} \frac{ \nu_1(u) \exp\left(-\frac{1}{\rho}\sum^{t-1}_{s=1}h_s(u)\right)}{\Ave{\nu_1, \exp\left(-\frac{1}{\rho}\sum^{t-1}_{s=1}h_s\right)}} \exp\left(- \frac{1}{\rho} h_t(u)\right) = \frac{\Ave{\nu_1, \exp\left(-\frac{1}{\rho}\sum^{t}_{s=1}h_s\right)}}{\Ave{\nu_1, \exp\left(-\frac{1}{\rho}\sum^{t-1}_{s=1}h_s\right)}}.
 $$
 We have
 \begin{align*} 
 &\rho \log \left( \sum_{u \in \sU}\exp\left(- \frac{1}{\rho} \sum^t_{s=1} h_s(u)\right)\right) \\
 &= \rho \log \Ave{\nu_1, \exp\left(- \frac{1}{\rho} \sum^t_{s=1} h_s\right)} + \rho \log | \sU |\\
 &= \rho \log \Ave{\nu_t,\exp\left(- \frac{1}{\rho} h_t\right)} + \rho \log \Ave{\nu_1, \exp\left(- \frac{1}{\rho} \sum^{t-1}_{s=1} h_s\right)}+ \rho \log | \sU | \\
&\le -\E_{U \sim \nu_t}h_t(U) + \frac{L(\sX,\sU, \cF)^2}{2\rho} + \rho \log \left( \sum_{u \in \sU}\exp\left(- \frac{1}{\rho} \sum^{t-1}_{s=1} h_s(u)\right)\right),
\end{align*}
where the first equality is due to the fact that $\nu_1$ is the uniform distribution on $\sU$, while the inequality is due to Hoeffding's lemma. Plugging the resulting bound into the admissibility condition, we get
\begin{align*} 
&\sup_{h_t \in \cH_x} \left\{\E_{U \sim \nu_t} \left[ h_t(U) \right] + \wh W_{x,t} (h^t)\right\} \\
&\le \rho \log \left( \sum_{u \in \sU}\exp\left(- \frac{1}{\rho} \sum^{t-1}_{s=1} h_s(u)\right)\right) + 2\frac{1}{\rho}(T-t+1) L(\sX,\sU, \cF)^2\\
&=\wh W_{x,t-1}(h^{t-1}).
\end{align*}
Thus, the recursive algorithm \eqref{eq:recursiveevendar} is admissible for the relaxation \eqref{eq:exprelax1}.

\section{Proof of Theorem~\ref{thm:evendarbound}}
\label{app:evendarbound}

Again, we drop the subscript $x$ and write $\nu_t$ for $P_t(\cdot|x)$, etc. We have
\begin{align} \label{eq:evendarregret_repeated}
		\E^{\wh{\bd{\gamma}},\bd{f}}_x\left[\sum^T_{t=1}f_t(X_t,U_t) - \Psi(\bd{f})\right] &\le \sup_{P \in \cM(\sU|\sX)}\sum_{x}\pi_P(x) \wh{W}_{x,0} + C_\cF  \sum^T_{t=1} \| \mu^{\wh{\bd{\gamma}},\bd{f}}_t - \pi^{\wh{\bd{\gamma}},\bd{f}}_t \|_1.
\end{align}
From the relaxation \eqref{eq:exprelax1}, it is easy to see $\wh{W}_{x,0} \le 2L\sqrt{2 T \log |\sU|}$ for all states $x$ (in fact, the bound is met with equality with the optimal choice of $\rho = \sqrt{\frac{2TL^2}{\log|\sU|}}$). Since we have bounded the first term, now we focus on bounding the second term of the regret bound. 

The relative entropy between $\nu_t$ and $\nu_{t-1}$ is given by
 \begin{align} 
D(\nu_t \| \nu_{t-1})&= \left\langle\nu_t, \log \frac{\exp\Big(-\frac{1}{\rho}\sum^{t-1}_{s=1} h_s\Big)}{\exp\Big(-\frac{1}{\rho}\sum^{t-2}_{s=1} h_s\Big)}\right\rangle + \log \frac{\Ave{\nu_1, \exp\left(- \frac{1}{\rho} \sum^{t-2}_{s=1} h_s\right)}}{\Ave{\nu_1, \exp\left(- \frac{1}{\rho} \sum^{t-1}_{s=1} h_s\right)}} \nonumber \\
&= -\frac{1}{\rho} \Ave{\nu_t, h_{t-1}} +\log \frac{\Ave{\nu_1, \exp\left(- \frac{1}{\rho} \sum^{t-2}_{s=1} h_s\right)}}{\Ave{\nu_1, \exp\left(- \frac{1}{\rho} \sum^{t-1}_{s=1} h_s\right)}}, \label{eq:recursivedistance}
\end{align}
where
\begin{align*} 
\frac{\Ave{\nu_1, \exp\left(- \frac{1}{\rho} \sum^{t-2}_{s=1} h_s\right)}}{\Ave{\nu_1, \exp\left(- \frac{1}{\rho} \sum^{t-1}_{s=1} h_s\right)}}&= \frac{\displaystyle\sum_{u \in \sU}\nu_1(u) \exp\left(- \frac{1}{\rho} \sum^{t-1}_{s=1} h_s(u)\right) \exp\left(\frac{1}{\rho}h_{t-1}(u)\right)}{\Ave{\nu_1, \exp\left(- \frac{1}{\rho} \sum^{t-1}_{s=1} h_s\right)}}\\
&= \Ave{\nu_t,\exp\left(\frac{1}{\rho}h_{t-1}\right)}.
\end{align*}
Using Hoeffding's lemma, we can write
\begin{align*}
 \log \frac{\Ave{\nu_1, \exp\left(- \frac{1}{\rho} \sum^{t-2}_{s=1} h_s\right)}}{\Ave{\nu_1, \exp\left(- \frac{1}{\rho} \sum^{t-1}_{s=1} h_s\right)}} \le   \frac{1}{\rho} \Ave{\nu_t,h_{t-1}} + \frac{L^2}{2\rho^2}
\end{align*}
Substituting this bound into \eqref{eq:recursivedistance}, we see that the terms involving the expectation of $h_{t-1}$ w.r.t.\ $\nu_t$ cancel, and we are left with
\begin{align*} 
D(\nu_t \| \nu_{t-1}) \le \frac{L^2}{2\rho^2}.
\end{align*}
Plugging in the optimal  value of $\rho$ and using Pinsker's inequality \cite{CoverThomas}, we find
\begin{align*} 
\| \nu_t - \nu_{t-1}\|_1  \le \sqrt{\frac{\log |\sU|}{2T}}.
\end{align*}
So far, we have been working with a fixed  state $x \in \sX$, so we had $\nu_t = P^{\wh{\bd{\gamma}},\bd{f}}_t(\cdot|x)$, where $\wh{\bd{\gamma}}$ is the agent's behavioral strategy induced by the relaxation \eqref{eq:exprelax1}. Since $x$ was arbitrary, we get the uniform bound
\begin{align}\label{eq:policy_difference}
	\max_{x \in \sX} \left\| P^{\wh{\bd{\gamma}},\bd{f}}_t(\cdot|x) - P^{\wh{\bd{\gamma}},\bd{f}}_{t-1}(\cdot|x) \right\|_1 \le \sqrt{\frac{\log |\sU|}{2T}}.
\end{align}
Armed with this estimate, we now bound the total variation distance  between the actual state distribution at time $t$ and the unique invariant distribution of $K^{\wh{\bd{\gamma}},\bd{f}}_t$. 

For any time $k \le t$, we have
\begin{align} 
\left\| \mu^{\wh {\bd \gamma}, \bd f}_k - \pi^{\wh {\bd \gamma}, \bd f}_t \right\|_1&= \left\| \mu^{\wh {\bd \gamma}, \bd f}_{k-1} K^{\wh {\bd \gamma}, \bd f}_{k-1} - \mu^{\wh {\bd \gamma}, \bd f}_{k-1} K^{\wh {\bd \gamma}, \bd f}_{t} + \mu^{\wh {\bd \gamma}, \bd f}_{k-1} K^{\wh {\bd \gamma}, \bd f}_{t}- \pi^{\wh {\bd \gamma}, \bd f}_t\right \|_1 \nonumber\\
&\stackrel{{\rm (a)}}{\le} \left \| \mu^{\wh {\bd \gamma}, \bd f}_{k-1} K^{\wh {\bd \gamma}, \bd f}_{t}- \pi^{\wh {\bd \gamma}, \bd f}_t \right\|_1 +\ \left\| \mu^{\wh {\bd \gamma}, \bd f}_{k-1} K^{\wh {\bd \gamma}, \bd f}_{k-1} - \mu^{\wh {\bd \gamma}, \bd f}_{k-1} K^{\wh {\bd \gamma}, \bd f}_{t} \right\|_1 \nonumber\\
&\stackrel{{\rm (b)}}{=}\left\| \mu^{\wh {\bd \gamma}, \bd f}_{k-1} K^{\wh {\bd \gamma}, \bd f}_{t}- \pi^{\wh {\bd \gamma}, \bd f}_t K^{\wh {\bd \gamma}, \bd f}_{t} \right\|_1+ \left\| \mu^{\wh {\bd \gamma}, \bd f}_{k-1} K^{\wh {\bd \gamma}, \bd f}_{k-1} - \mu^{\wh {\bd \gamma}, \bd f}_{k-1} K^{\wh {\bd \gamma}, \bd f}_{t} \right\|_1 \nonumber\\
&\stackrel{{\rm (c)}}{\le} e^{-1/\tau} \left\| \mu^{\wh {\bd \gamma}, \bd f}_{k-1} - \pi^{\wh {\bd \gamma}, \bd f}_t \right\|_1+ \max_{x \in \sX}\left\| P^{\wh{\bd{\gamma}},\bd{f}}_{k-1}(\cdot|x) - P^{\wh{\bd{\gamma}}}_t(\cdot|x) \right\|_1 \nonumber\\
&\stackrel{{\rm (d)}}{\le} e^{-1/\tau}\left\| \mu^{\wh {\bd \gamma}, \bd f}_{k-1} - \pi^{\wh {\bd \gamma}, \bd f}_t \right\|_1+ \sum^{t-1}_{i=k-1}\sqrt{\frac{\log |\sU|}{2T}}, \label{eq:k_to_t}
\end{align}
where (a) is by triangle inequality; (b) is by invariance of $\pi^{\wh{\bd\gamma},\bd{f}}_t$ w.r.t.\ $K^{\wh{\bd\gamma},{\bd f}}_t$; (c) is by the uniform mixing bound \eqref{eq:mixingtimeeq}; and (d) follows from repeatedly using \eqref{eq:policy_difference} together with triangle inequality and the easily proved fact that, for any state distribution $\mu \in \cP(\sX)$ and any two Markov kernels $P,P' \in \cM(\sU|\sX)$,
\begin{align*} 
	\left\| \mu K(\cdot|P) - \mu K(\cdot|P') \right\|_1 \le \max_{x \in \sX} \left\| P(\cdot|x) - P'(\cdot|x) \right\|_1.
\end{align*}
Letting now the initial state distribution be $\mu_1$, we can apply the bound \eqref{eq:k_to_t} recursively to obtain
\begin{align*} 
\left\| \mu^{\wh {\bd \gamma}, \bd f}_t - \pi^{\wh {\bd \gamma}, \bd f}_t \right\|_1 &\le e^{-(t-1)/\tau}\left\| \mu_1 - \pi^{\wh {\bd \gamma}, \bd f}_t \right\|_1+ \sum^t_{k=2} e^{-\frac{t-k}{\tau}}\sum^{t}_{i=k-1}\sqrt{\frac{\log |\sU|}{2T}}\\
&\le 2 e^{-(t-1)/\tau} + \sum^t_{k=2} e^{-\frac{t-k}{\tau}}(t-k+1)\sqrt{\frac{\log |\sU|}{2T}} \\
&\le 2 e^{-(t-1)/\tau} + \sqrt{\frac{\log |\sU|}{2T}}\sum^{\infty}_{k=0} (k+1) e^{-\frac{k}{\tau}} \\
&\le 2 e^{-(t-1)/\tau} + (\tau+1)^2\sqrt{\frac{\log |\sU|}{2T}}.
\end{align*}
So, the second term on the right-hand side of \eqref{eq:evendarregret_repeated} can bounded by
\begin{align*} 
C_{\cF} \sum^T_{t=1} \| \mu^{\wh {\bd \gamma}, \bd f}_t - \pi^{\wh {\bd \gamma}, \bd f}_t \|_1 \le C_{\cF}(\tau+1)^2 \sqrt{\frac{\log|\sU|T}{2}} + (2\tau+2) C_{\cF},
\end{align*}
which completes the proof.

\section{Proof of Proposition~\ref{pps:recoveryu}}
\label{app:recoveryu}

First we show that the relaxation \eqref{eq:modrelax} arises as an upper bound on the conditional sequential Rademacher complexity. Once again, we omit the subscript $x$ from $h_{x,t}$ etc.\ to keep the notation light. Following the same steps as in Appendix~\ref{app:recoverevendar}, we have, for any $\rho > 0$,
\begin{align*}
&  \E_{\eps} \left[ \max_{u \in \sU} \left\{ 2\sum^M_{j=m+1}\eps_j \sum_{t \in \cT_j}\left[\mathbf h_{t}(\eps)\right](u) - \sum^m_{i=1} \sum_{s \in \cT_i} h_s(u) \right\} \right]\\ 
&\le \rho \log \left( \E_{\eps} \left[ \max_{u \in \sU} \exp\left( \frac{2}{\rho} \sum^M_{j=m+1}\eps_j \sum_{t \in \cT_j}\left[\mathbf h_{t}(\eps)\right](u) - \frac{1}{\rho} \sum^m_{i=1} \sum_{s \in \cT_i} h_s(u) \right) \right]\right)\\
&\le \rho \log \left( \E_{\eps} \left[ \sum_{u \in \sU} \exp\left( \frac{2}{\rho} \sum^M_{j=m+1}\eps_j \sum_{t \in \cT_j}\left[\mathbf h_{t}(\eps)\right](u) - \frac{1}{\rho} \sum^m_{i=1} \sum_{s \in \cT_i} h_s(u) \right) \right]\right).
\end{align*} 
In the same vein, 
\begin{align*}
&\rho \log \left( \sum_{u \in \sU}\exp\left(- \frac{1}{\rho}\sum^m_{i=1} \sum_{s \in \cT_i} h_s(u)\right)\E_{\eps} \left[  \prod^{M}_{j=m+1} \exp\left( \frac{2}{\rho} \eps_j \sum_{t \in \cT_j}\left[\mathbf h_{t}(\eps)\right](u)  \right) \right]\right)\\
&\le \rho \log \left( \sum_{u \in \sU}\exp\left(- \frac{1}{\rho} \sum^m_{i=1} \sum_{s \in \cT_i} h_s(u)\right)\times  \exp\left( \frac{2}{\rho^2} \max_{\eps_{m+1}, \ldots, \eps_{M} \in \{\pm1\}}\sum^{M}_{j=m+1} \left(\tau_j \left[\mathbf h(\eps)\right](u) \right)^2 \right) \right)\\
&\le \rho \log \left( \sum_{u \in \sU}\exp\left(- \frac{1}{\rho} \sum^m_{i=1} \sum_{s \in \cT_i} h_s(u)\right)\max_{u}  \exp\left( \frac{2}{\rho^2} \max_{\eps_{m+1}, \ldots, \eps_{M} \in \{\pm1\}}\sum^{M}_{j=m+1} \left(\tau_j \left[\mathbf h(\eps\right](u) \right)^2 \right) \right) \\
&\le \rho \log \left( \sum_{u \in \sU}\exp\left(- \frac{1}{\rho} \sum^m_{i=1} \sum_{s \in \cT_i} h_s(u)\right)\right) +\frac{2}{\rho} \sup_{\mathbf h}\max_{u \in \sU} \max_{\eps_{m+1}, \ldots, \eps_{M} \in \{\pm1\}}\sum^{M}_{j=m+1} \left(\tau_j \left[\mathbf h(\eps)\right](u) \right)^2 \\
&\le \rho \log \left( \sum_{u \in \sU}\exp\left(- \frac{1}{\rho} \sum^m_{i=1} \sum_{s \in \cT_i} h_s(u)\right)\right) +\frac{2}{\rho} \sum^{M}_{j=m+1} \tau_j^2 L(\sX,\sU, \cF)^2,
\end{align*} 
where the first inequality is due to Hoeffding's lemma, while the last inequality is by Assumption~\ref{as:existboundQ}. We thus derive the relaxation in \eqref{eq:modrelax}.

Now we prove that this relaxation is admissible, and leads to the lazy algorithm \eqref{eq:recursiveyu} 
The admissibility condition to be proved is
\begin{align*} 
\sup_{h_m \in \cH^{\tau_m}_x} \left\{\E_{U \sim \nu_m} \left[ \sum_{s \in \cT_m}h_s(U) \right] + \wh W_{x,m} (h_{1:m})\right\}  \le \wh W_{x,m-1}(h_{1:m-1}),
\end{align*}
where $\nu_m = P_{m}(\cdot|x)$ is the Markov policy used in phase $m$. We have
 \begin{align*} 
 &\rho \log \left( \sum_{u \in \sU}\exp\left(- \frac{1}{\rho} \sum^m_{i=1} \sum_{s \in \cT_i} h_s(u)\right)\right) \\
 &= \rho \log \Ave{\nu_1, \exp\left(- \frac{1}{\rho} \sum^m_{i=1} \sum_{s \in \cT_i} h_s\right)} + \rho \log | \sU |\\
 &= \rho \log  \Ave{\nu_m,\exp\left(- \frac{1}{\rho} \sum_{s \in \cT_m}h_s\right)} + \rho \log \Ave{\nu_1, \exp\left(- \frac{1}{\rho} \sum^{m-1}_{i=1}\sum_{s \in \cT_i} h_s\right)}+ \rho \log | \sU | \\
&\le -\E_{U \sim \nu_m} \left[ \sum_{s \in \cT_m}h_s(U) \right] + \frac{\tau_m^2L(\sX,\sU, \cF)^2}{2\rho} + \rho \log \left( \sum_{u \in \sU}\exp\left(- \frac{1}{\rho} \sum^{m-1}_{i=1}\sum_{s \in \cT_i} h_s(u)\right)\right),
\end{align*}
Plugging this into the admissibility condition, we have
\begin{align*} 
&\sup_{h_m \in \cH^{\tau_m}_x} \left\{\E_{U \sim \nu_m} \left[ \sum_{s \in \cT_m}h_s(U) \right] + \wh W_{x,m} (h_{1:m})\right\} \\
&\le \rho \log \left( \sum_{u \in \sU}\exp\left(- \frac{1}{\rho} \sum^{m-1}_{i=1}\sum_{s \in \cT_i}  h_s(u)\right)\right) + \frac{2}{\rho} \sum^{M}_{j=m+1} \tau_j^2 L(\sX,\sU, \cF)^2 + \frac{\tau_m^2L(\sX,\sU, \cF)^2}{2\rho}\\
&\le \rho \log \left( \sum_{u \in \sU}\exp\left(- \frac{1}{\rho} \sum^{m-1}_{i=1}\sum_{s \in \cT_i}  h_s(u)\right)\right) + \frac{2}{\rho} \sum^{M}_{j=m} \tau_j^2 L(\sX,\sU, \cF)^2\\
&=\wh W_{x,m-1}(h^{m-1}).
\end{align*}
So the lazy algorithm \eqref{eq:recursiveyu} is an admissible strategy for the relaxation \eqref{eq:modrelax}. 

\section{Proof of Theorem~\ref{thm:yubound}}
\label{app:yubound}

The state feedback law $P^{\wh{\bd{\gamma}},\bd{f}}_t(\cdot|x)$ that the agent applies within phase $m$ is the same for all $t \in \cT_m$, and we denote it by $P^{\wh{\bd{\gamma}},\bd{f}}_m(\cdot|x)$. Let $K^{\wh {\bd \gamma},\bd{f}}_m$ denote the Markov matrix that describes the state transition from $X_t$ to $X_{t+1}$ if $t \in \cT_m$.  Thus, we can write
\begin{align*}
K^{\wh {\bd \gamma},\bd{f}}_m(y|x) = \sum_u K(y|x,u) P^{\wh{\bd{\gamma}},\bd{f}}_t(u|x), \qquad \forall x,y \in \sX.
\end{align*}
First, we show that 
\begin{align} \label{eq:regretboundinphase}
\E^{\wh {\bd \gamma}, \bd f}_x \left[ \sum^T_{t=1}f_t(X_t,U_t) - \Psi(\bd f)\right] \le \sup_{P \in \cM(\sU | \sX)} \sum_x \pi_P(x) \wh W_{x,0} + C_{\cF} \sum^M_{m=1} \sum_{t \in \cT_m}\| \mu^{\wh {\bd \gamma}, \bd f}_t - \pi^{\wh {\bd \gamma}, \bd f}_m \|_1,
\end{align}
where $\pi^{\wh {\bd \gamma}, \bd f}_m$ is the invariant distribution of $K^{\wh {\bd \gamma},\bd{f}}_m$. 

To prove \eqref{eq:regretboundinphase}, we write
\begin{align*}
&	\E^{\wh {\bd \gamma}, \bd f}_x \left[ \sum^T_{t=1}f_t(X_t,U_t) - \Psi(\bd f)\right]\\ 
&\le \sup_{P \in \cM(\sU | \sX)} \sum^T_{t=1} \left[ \ave{\pi^{\wh {\bd \gamma}, \bd f}_t \otimes P^{\wh {\bd \gamma}, \bd f}_t,f_t} - \ave{\pi_P \otimes P, f_t}\right] + \sum^T_{t=1} \| f_t \|_{\infty} \| \mu^{\wh {\bd \gamma}, \bd f}_t - \pi^{\wh {\bd \gamma}, \bd f}_t \|_1\\
&\le \sup_{P \in \cM(\sU | \sX)} \sum^M_{m=1} \sum_{t \in \cT_m}\left[ \ave{\pi^{\wh {\bd \gamma}, \bd f}_t \otimes P^{\wh {\bd \gamma}, \bd f}_t,f_t} - \ave{\pi_P \otimes P, f_t}\right] + C_{\cF} \sum^M_{m=1} \sum_{t \in \cT_m}\| \mu^{\wh {\bd \gamma}, \bd f}_t - \pi^{\wh {\bd \gamma}, \bd f}_m \|_1\\
&\le \sup_{P \in \cM(\sU | \sX)} \sum_x \pi_P(x) \sum^M_{m=1} \sum_{t \in \cT_m} \left( \sum_u P^{\wh {\bd \gamma}, \bd f}_m(u|x) \wh Q^{\wh {\bd \gamma}, \bd f}_t(x,u) - P(u|x) \wh Q^{\wh {\bd \gamma}, \bd f}_t(x,u)\right) + C_{\cF} \sum^M_{m=1} \sum_{t \in \cT_m}\| \mu^{\wh {\bd \gamma}, \bd f}_t - \pi^{\wh {\bd \gamma}, \bd f}_m \|_1,
\end{align*} 
where the last inequality is by Lemma~\ref{lem:comparisonch3}. By writing out the first term in the right hand side, we get 
\begin{align*}
& \sup_{P \in \cM(\sU | \sX)} \sum_x \pi_P(x) \Bigg[\sum^{M-1}_{m=1} \sum_{t \in \cT_m} \left( \sum_u P^{\wh {\bd \gamma}, \bd f}_m(u|x) \wh Q^{\wh {\bd \gamma}, \bd f}_t(x,u) \right) + \sum_u \nu_M(u|x) \sum_{t \in \cT_M}\wh Q^{\wh {\bd \gamma}, \bd f}_t(x,u) \\
&\quad- \sum^T_{t=1}P(u|x) \wh Q^{\wh {\bd \gamma}, \bd f}_t(x,u) \Bigg]+ C_{\cF} \sum^M_{m=1} \sum_{t \in \cT_m}\| \mu^{\wh {\bd \gamma}, \bd f}_t - \pi^{\wh {\bd \gamma}, \bd f}_m \|_1\\
&\le \sup_{P \in \cM(\sU | \sX)} \sum_x \pi_P(x) \left[\sum^{M-1}_{m=1} \sum_{t \in \cT_m} \left( \sum_u P^{\wh {\bd \gamma}, \bd f}_m(u|x) \wh Q^{\wh {\bd \gamma}, \bd f}_t(x,u) \right) + \sum_u \nu_M(u|x) \sum_{t \in \cT_M}\wh Q^{\wh {\bd \gamma}, \bd f}_t(x,u)+\wh W_{x,M} (h^M) \right]\\
&\quad+ C_{\cF} \sum^M_{m=1} \sum_{t \in \cT_m}\| \mu^{\wh {\bd \gamma}, \bd f}_t - \pi^{\wh {\bd \gamma}, \bd f}_m \|_1\\
&\le \sup_{P \in \cM(\sU | \sX)} \sum_x \pi_P(x) \left[\sum^{M-1}_{m=1} \sum_{t \in \cT_m} \left( \sum_u P^{\wh {\bd \gamma}, \bd f}_m(u|x) \wh Q^{\wh {\bd \gamma}, \bd f}_t(x,u) \right) + \wh W_{x,M-1} (h^{M-1}) \right]  \\
& \qquad \qquad \qquad 
+\sum^T_{t=1} \| f_t \|_{\infty} \left\| \mu^{\wh {\bd \gamma}, \bd f}_t - \pi^{\wh {\bd \gamma}, \bd f}_t \right\|_1.
\end{align*} 
The last inequality is due to the fact that $\wh{\bd{\gamma}}$ is the behavioral strategy associated to the admissible relaxation $\{\wh W_{x,m}\}^M_{m=1}$. Continuing this induction backwards, we arrive at \eqref{eq:regretboundinphase}.

Next, we bound the two terms on the right-hand side of \eqref{eq:regretboundinphase}. From the form of the relaxation \eqref{eq:modrelax}, it is easy to see $\wh{W}_{x,0} \le 2L\sqrt{2 \log |\sU| \sum^M_{i=1}\tau_i^2}$ for all states $x$; in fact, this bound is attained with  equality if we use the optimal choice $\rho = \sqrt{\frac{2\sum^M_{i=1}\tau_i^2L^2}{\log|\sU|}}$. Since we have bounded the first term, now we focus on bounding the second term of \eqref{eq:regretboundinphase}. 

From the contraction inequality \eqref{eq:mixingtimeeq} it follows that, for every $k \in \{0,1,\ldots, \tau_m-1\}$, we have
\begin{align*}
	\left\| \mu^{\wh {\bd \gamma}, \bd f}_{\tau_{1:m-1}+k+1} - \pi^{\wh {\bd \gamma}, \bd f}_m \right\|_1 
	&= \left\| \mu^{\wh {\bd \gamma}, \bd f}_{\tau_{1:m-1}+1} {(K^{\wh {\bd \gamma},\bd{f}}_m)}^k- \pi^{\wh {\bd \gamma}, \bd f}_m {(K^{\wh {\bd \gamma},\bd{f}}_m)}^k  \right\|_1 \\
	&\le e^{-k/\tau} \left\| \mu^{\wh {\bd \gamma}, \bd f}_{\tau_{1:m-1}+1} - \pi^{\wh {\bd \gamma}, \bd f}_m  \right\|_1 \\
	&\le 2 e^{-k/\tau}.
\end{align*}
Hence,
\begin{align*}
	\sum_{t \in \cT_m}\left\| \mu^{\wh {\bd \gamma}, \bd f}_t - \pi^{\wh {\bd \gamma}, \bd f}_m \right\|_1 \le 2\displaystyle \sum_{k=0}^{\tau_m -1} e^{-k/\tau} \le \frac {2}{1-e^{-1/\tau}}.
\end{align*}
Plugging it in \eqref{eq:regretboundinphase}, we have shown that 
\begin{align*}
\E^{\wh{\bd{\gamma}},\bd{f}}_x \left[\sum^T_{t=1}f_t(X_t,U_t) - \Psi(\bd{f})\right] \le 2 L\sqrt{2 \log |\sU| \sum^M_{i=1}\tau_i^2} +  \frac {2C_\cF M}{1-e^{-1/\tau}}.
\end{align*}

\section{Proof of Corollary~\ref{cor:optbound}}
\label{app:optbound}
Let us inspect the right-hand side of \eqref{eq:yuregret}. We see that both $\sqrt{\sum^M_{j=1}\tau^2_j}$ and $M$ have to be sublinear in $T$. Since $\sum^M_{i=1}\tau_i = T$ and $\sqrt{\sum^M_{i=1}\tau_i^2} < \sqrt{(\sum^M_{i=1}\tau_i)^2}$, at least the first of these terms can be made sublinear, e.g., by having $\tau_j=1$ for all $j$. Of course, this means that $M=T$, so we need longer phases. For example, if we follow \cite{Yu} and let $\tau_m = \lceil m^{1/3-\eps} \rceil$ for some $\eps \in (0,1/3)$, then a straightforward if tedious algebraic calculation shows that $M = O(T^{3/4})$ and $\sqrt{\sum^M_{j=1}\tau^2_j} = O(T^{5/8})$, which yields the regret of $O(T^{3/4})$.

However, if $T$ is known in advance, then we can do better: ignoring the rounding issues, for any constants $A_1,A_2 > 0$,
\begin{align}\label{eq:best_phases}
	\min_{1 \le M \le T}\min \left\{ A_1 \sqrt{\sum^M_{j=1}\tau^2_j} + A_2 M : \sum^M_{j=1}\tau_j = T\right\} = O(T^{2/3}),
\end{align}
To see this, let us first fix $M$ and optimize the choice of the $\tau_j$'s:
$$
\min \sum^M_{j=1} \tau_j^2 \quad  \text{subject to }   \sum^M_{j=1} \tau_j = T.
$$
By the Cauchy--Schwarz inequality, we have
$$
\sum^M_{j=1} \tau_j \le \sqrt{M\sum^M_{j=1} \tau_j^2}.
$$
Thus, $\sum^M_{j=1} \tau_j^2$ achieves its minimum when the above bound is met with equality. This will happen only if all the $\tau_j$'s are equal, i.e., $\tau_j = \frac{T}{M}$ for every $j$ (for simplicity, we assume that $M$ divides $T$ --- otherwise, the remainder term will be strictly smaller than $M$, and the bound in \eqref{eq:best_phases} will still hold, but with a larger multiplicative constant). Therefore,
\begin{align*}
		\min_{1 \le M \le T}\min \left\{ A_1 \sqrt{\sum^M_{j=1}\tau^2_j} + A_2 M : \sum^M_{j=1}\tau_j = T\right\} = \min_{1 \le M \le T} \left( \frac{A_1T}{\sqrt{M}} + A_2 M\right) = O(T^{2/3}),
\end{align*}
where the minimum on the right-hand side (again, ignoring rounding issues) is achieved by $M = T^{2/3}$ and $\tau_j = T^{1/3}$ for all $j$. This shows that, for a given horizon $T$, the optimal choice of phase lengths is $T^{1/3}$, which gives the regret of $O(T^{2/3})$, better than the $O(T^{3/4})$ bound derived by \cite{Yu}.

\section{Proof of Proposition~\ref{pps:OMDrelaxadmin}}
\label{app:OMDrelaxadmin}

First, we check the admissibility condition at time $t=T$. Since the Bregman divergence is nonnegative, we have
\begin{align*} 
\wh{V}_T({\cal G}'|f_1,\ldots,f_T) = &   \sup_{\mu \in {\cal G}'} \left\{\sum^T_{s=1} \ave{\mu,-f_s} + \frac{1}{\rho} D_{F} (\mu, \nu_{T+1})\right\}   \\
& \ge  -\inf_{\mu \in {\cal G}'}  \Ave{\mu, \sum^T_{s=1}f_s}.
\end{align*}
Now let us consider an arbitrary $t$. From the construction of our relaxation, we have
\begin{align*} 
 &\sup_{f_t \in \cF} \left\{ \ave{\nu_t,  f_t} + \wh{V}_T({\cal G}|f_1,\ldots,f_t) \right\} \\
&=  \sup_{f_t \in \cF} \sup_{\mu \in {\cal G}'} \left\{ \sum^{t-1}_{s=1} \ave{\mu,-f_s} + \ave{\nu_t - \mu, f_t} + \frac{1}{\rho} D_{F} (\mu, \nu_{t+1}) + \frac{1}{2\rho} (T-t)  \right\}
\end{align*}
for all $t = 1,\ldots,T$. From the definition \eqref{eq:Bregman} of the Bregman divergence, the following equality holds for any three $\mu,\nu,\lambda \in {\mathcal G}$: 
\begin{align}\label{eq:Bregman_triple}
D_{F}(\mu,\nu) + D_{F}(\nu,\lambda)= D_{F}(\mu,\lambda) + \Ave{\nabla F(\lambda)-\nabla F(\nu),\mu - \nu}.
\end{align}
Since $\nabla F$ and $\nabla F^*$ are inverses of each other, we have $-\rho f_t = \nabla  F(\nu_{t+1}) - \nabla F(\nu_t)$. Using this fact together with \eqref{eq:Bregman_triple}, for any $\mu \in {\mathcal G}$ we can write
\begin{align} 
\Ave{\nu_t - \mu, \rho f_t} &= \Ave{ \nabla F(\nu_{t+1}) - \nabla F(\nu_t), \mu- \nu_t} \nonumber\\
&= D_{F}(\mu,\nu_t) - D_{F}(\mu,\nu_{t+1}) + D_{F}(\nu_t, \nu_{t+1}). \label {eq:md1}
\end{align}
Moreover, once again using the fact that $\nabla F$ and $\nabla F^*$ are inverses of one another, we have
\begin{align} 
D_{F}(\nu_t, \nu_{t+1}) &= D_{F^*}(\nabla F(\nu_{t+1}),\nabla F(\nu_t) ) \nonumber \\
&= F^*(\nabla F(\nu_{t+1})) - F^*(\nabla F(\nu_{t})) - \Ave{\nabla F^*(\nabla F(\nu_{t})), \nabla F(\nu_{t+1}) - \nabla F(\nu_t)} \nonumber \\
&\le \Lambda \left( \rho \| f_t \|^*_{\nabla F(\nu_t)}\right), \label {eq:md2}
\end{align}
where
$$
\Lambda(r) \deq -\log(1-r)-r = \frac{r^2}{2} + \frac{r^3}{3} + \frac{r^4}{4} + \ldots.
$$
Note that, because of the definition of $\Lambda$, the learning rate $\rho$ must be chosen in such a way that $\rho \| f_t \|^*_{\nabla F(\mu_t)} < 1$ for all $t=1,\ldots,T$. By hypothesis, we have $\rho \| f_t \|^*_{\nabla F(\nu_t)} \le 1/2$ for all $t$. The first line of Eq.~(\ref{eq:md2}) is by Prop.~11.1 in \cite{PLG}, the second is by definition of the Bregman divergence, and the third follows from from a local second-order Taylor formula for a self-concordant function \cite[Eq.~(2.5)]{NemTodd} (which is applicable because, by hypothesis, $\rho \| f_t \|^*_{\nabla F(\nu_t)} \le 1/2 < 1$ for all $t$) and the fact that the Legendre--Fenchel dual of a self-concodrant function is also self-concordant.

Using the inequality $\log r \le r-1$, we can upper-bound
$$
\Lambda(r) = -\log(1-r)-r \le \frac{1}{1-r}-1-r = \frac{1-(1-r)(1+r)}{1-r} = \frac{r^2}{1-r}.
$$
Moreover, since $\rho \| f_t \|^*_{\nabla F(\nu_t)} \le 1/2$ and $\| f_t \|^*_{\nabla F(\nu_t)} \le 1$for all $t \in \{1,\ldots,T\}$ by hypothesis, we can further bound
$$
\Lambda\left(\rho \| f_t \|^*_{\nabla F(\nu_t)}\right) \le 2\rho^2 \| f_t \|^{*2}_{\nabla F(\nu_t)} \le 2\rho^2 .
$$

Applying Eqs.~\eqref{eq:md1} and \eqref{eq:md2}, we arrive at
\begin{align} 
 &\sup_{f_t \in \cF} \left\{ \ave{\nu_t,  f_t} + \wh{V}_T({\cal G}'|f_1,\ldots,f_t) \right\} \nonumber\\
&= \sup_{f_t \in \cF} \sup_{\mu \in {\cal G}'} \left\{ \sum^{t-1}_{s=1} \ave{\mu,-f_s} + \ave{\nu_t - \mu, f_t} + \frac{1}{\rho} D_{F} (\mu, \nu_{t+1}) + 2\rho (T-t) \right\}  \nonumber\\
&\le \sup_{f_t \in \cF} \sup_{\mu \in {\cal G}'} \left\{ \sum^{t-1}_{s=1} \ave{\mu,-f_s} + \frac{1}{\rho} D_{F} (\mu, \nu_{t}) + \frac{1}{\rho} \Lambda \left( \rho \| f_t \|^*_{\nabla F(\nu_t)}\right) + 2\rho (T-t) \right\} \nonumber\\
&\le  \sup_{f_t \in \cF} \sup_{\mu \in {\cal G}'} \left\{ \sum^{t-1}_{s=1} \ave{\mu,-f_s} + \frac{1}{\rho} D_{F} (\mu, \nu_{t}) + 2\rho (T-t+1) \right\} \nonumber\\
&=\wh{V}_T({\cal G}'|f_1,\ldots,f_{t-1}) \nonumber.
\end{align}
This shows that the proposed algorithm (behavoiral strategy) is admissible, and the proof is complete.

\section{Proof of Theorem~\ref{thm:OMDbound1}}
\label{app:OMDbound1}
Since the relaxation \eqref{eq:OMDrelax} is admissible by Proposition~\ref{pps:OMDrelaxadmin}, we have 
\begin{align} 
		\sum^T_{t=1} \ave{\nu_t,f_t} - \inf_{\nu \in {\cal G}'} \sum^T_{t=1} \ave{\nu,f_t} &\le \wh{V}_T({\cal G}' | {\mathsf e}) \nonumber\\
		&= \sup_{\mu \in {\cal G}'} \left\{  \frac{D_F(\mu,\nu_1)}{\rho}  + 2\rho T	\right\} \nonumber\\
		&= \frac{D_F({\cal G}')}{\rho}  + 2\rho T \nonumber. 
\end{align}

\section{Proof of Theorem~\ref{thm:OMDbound2}}
\label{app:OMDbound2}
Let us denote by $P_t = P_{\nu_t}$ the policy extracted from $\nu_t$, and the induced marginal distribution of $X_t$ by $\mu_t \in \cP(\sX)$. We also denote by $K_t$ the Markov matrix that describes the state transition from $X_t$ to $X_{t+1}$, and by $\pi_t \in \cP(\sX)$ its the unique invariant distribution. Finally, we denote by $\wh{\bd{\gamma}}$ the behavioral strategy corresponding to our algorithm. Then, for any $\bd{f} \in \cF^T$, we can upper-bound the regret by 
\begin{align} \label{eq:regretboundOMD}
	R^{\wh{\bd{\gamma}},\bd{f}}_x({\cal G'}) &= \sum^T_{t=1} \Ave{\mu_t \otimes P_t, f_t} - \inf_{\nu \in {\cal G}'} \sum^T_{t=1} \ave{\nu,f_t} \nonumber\\
	&\le \sum^T_{t=1}\left[\ave{\pi_t \otimes P_t,f_t} -  \inf_{P \in \cM({\cal G}')}\, \ave{\pi_P \otimes P, f_t}\right]  + \sum^T_{t=1} \| f_t \|_\infty \| \mu_t - \pi_t \|_1 \nonumber\\
	&=\sum^T_{t=1} \Ave{\nu_t, f_t} - \inf_{\nu \in {\cal G}'} \sum^T_{t=1} \ave{\nu, f_t} +\sum^T_{t=1} \| f_t \|_\infty \| \mu_t - \pi_t \|_1 \nonumber\\
	&\le  \frac{D_F({\cal G}')}{\rho}  + 2\rho T + \sum^T_{t=1} \| f_t \|_\infty \| \mu_t - \pi_t \|_1.
	\end{align}
Now we focus on bounding the third term of the regret bound. For any time $k \le t$, we have
\begin{align} 
\left\| \mu_k - \pi_t \right\|_1&= \left\| \mu_{k-1} K_{k-1} - \mu_{k-1} K_{t} + \mu_{k-1} K_{t}- \pi_t\right \|_1 \nonumber\\
&\stackrel{{\rm (a)}}{\le} \left \| \mu_{k-1} K_{t}- \pi_t \right\|_1 +\ \left\| \mu_{k-1} K_{k-1} - \mu_{k-1} K_{t} \right\|_1 \nonumber\\
&\stackrel{{\rm (b)}}{=}\left\| \mu_{k-1} K_{t}- \pi_t K_{t} \right\|_1+ \left\| \mu_{k-1} K_{k-1} - \mu_{k-1} K_{t} \right\|_1 \nonumber\\
&\stackrel{{\rm (c)}}{\le} e^{-1/\tau} \left\| \mu_{k-1} - \pi_t \right\|_1+ \max_{x \in \sX}\left\| P_{k-1}(\cdot|x) - P_t(\cdot|x) \right\|_1,\nonumber\\
&\stackrel{{\rm (d)}}{\le} e^{-1/\tau}\left\| \mu_{k-1} - \pi_t \right\|_1+ \sum^{t-1}_{j=k-1} \max_{x \in \sX} \| P_j(\cdot|x) - P_{j+1}(\cdot|x) \|_1,  \label{eq:k_to_t_OMD}
\end{align}
where (a) is by triangle inequality; (b) is by invariance of $\pi_t$ w.r.t.\ $K^{\wh{\bd\gamma},{\bd f}}_t$; and (c) follows from the uniform mixing bound \eqref{eq:mixingtimeeq}, and (d) follows from the triangle inequality and the easily proved fact that, for any state distribution $\mu \in \cP(\sX)$ and any two Markov kernels $P,P' \in \cM(\sU|\sX)$,
\begin{align*} 
	\left\| \mu K(\cdot|P) - \mu K(\cdot|P') \right\|_1 \le \max_{x \in \sX} \left\| P(\cdot|x) - P'(\cdot|x) \right\|_1.
\end{align*}
Letting now the initial state distribution be $\mu_1$, we can apply the bound \eqref{eq:k_to_t_OMD} recursively to obtain
\begin{align*} 
\left\| \mu_t - \pi_t \right\|_1 &\le e^{-(t-1)/\tau}\left\| \mu_1 - \pi_t \right\|_1+ \sum^t_{k=2} e^{-\frac{t-k}{\tau}}\sum^{t-1}_{j=k-1} \max_{x \in \sX} \|P_{j}(\cdot|x)-P_{j+1}(\cdot|x)\|_1 \\
&\le 2 e^{-(t-1)/\tau} + \sum^t_{k=2} e^{-\frac{t-k}{\tau}}(t-k+1)\Delta_T \\
&\le 2 e^{-(t-1)/\tau} + B\sum^{\infty}_{k=0} (k+1) e^{-\frac{k}{\tau}} \\
&\le 2 e^{-(t-1)/\tau} + (\tau+1)^2\Delta_T.
\end{align*}
So, the second term on the right-hand side of \eqref{eq:regretboundOMD} can bounded by
\begin{align*} 
C_{\cF} \sum^T_{t=1} \| \mu^{\wh {\bd \gamma}, \bd f}_t - \pi^{\wh {\bd \gamma}, \bd f}_t \|_1 \le C_{\cF}(\tau+1)^2 T \Delta_T + (2\tau+2) C_{\cF},
\end{align*}
which completes the proof.


\end{document}